\documentclass{article}

\usepackage[main,final]{neurips_2026}

\usepackage[utf8]{inputenc}
\usepackage[T1]{fontenc}
\usepackage{hyperref}
\usepackage{url}
\usepackage{booktabs}
\usepackage{amsfonts}
\usepackage{nicefrac}
\usepackage{microtype}
\usepackage{xcolor}
\usepackage{amsmath}
\usepackage{amssymb}
\usepackage{mathtools}
\usepackage{amsthm}
\usepackage{algorithm}
\usepackage{algorithmic}
\usepackage{multirow}
\usepackage{float}
\usepackage[capitalize,noabbrev]{cleveref}
\usepackage{enumitem}
\usepackage{wrapfig}
\usepackage{caption}
\usepackage{tikz}
\usetikzlibrary{arrows.meta}
\usepackage{placeins}
\usepackage{tcolorbox}
\newtcolorbox{insightbox}{colback=black!5,colframe=black!75,boxrule=0.8pt,arc=2mm,boxsep=0pt,left=5pt,right=5pt,top=3pt,bottom=3pt,before skip=5pt,after skip=7pt}

\theoremstyle{plain}
\newtheorem{theorem}{Theorem}[section]
\newtheorem{proposition}[theorem]{Proposition}
\newtheorem{lemma}[theorem]{Lemma}

\theoremstyle{definition}

\theoremstyle{remark}

\newcommand{\ve}{\mathbf{e}}

\title{Distilling Graph Geometry: \\Knowledge Gap from GNNs to MLPs\rlap{\hspace{0.45em}\smash{\raisebox{-1.5pt}{\includegraphics[width=50pt]{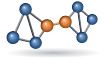}}}}}

\author{%
  Zhewei Chen$^{1}$\quad Hao Zhu$^{2}$\thanks{Corresponding authors.}\hphantom{$^{*}$}\quad Jiaojiao Jiang$^{3}$\quad Ahad N. Zehmakan$^{1}$\footnotemark[1]\hphantom{$^{*}$}\\
  $^{1}$Australian National University\quad
  $^{2}$CSIRO Technology\quad
  $^{3}$University of New South Wales\\
  {\small\texttt{\{zhewei.chen,\,ahadn.zehmakan\}@anu.edu.au}}\\
  {\small\texttt{allenhaozhu@gmail.com}\quad
  \texttt{jiaojiao.jiang@unsw.edu.au}}
}

\begin{document}

\maketitle

\begin{abstract}
GNN-to-MLP distillation aims to retain the predictive accuracy of a
message-passing teacher while deploying a graph-free MLP at inference.
Existing methods mainly transfer node-wise predictions or use
confidence-based reweighting, but they do not specify where the
student should preserve the teacher's graph-induced geometry. We show
that this omission leads to two spectral failure modes in the
student's representation space. On sparse graphs, the student suffers
from \emph{spectral underfit}, missing high-energy teacher directions
induced by aggregation. On dense graphs, it suffers from
\emph{spectral overfit}, retaining spurious directions that the
teacher has collapsed through aggregation. Motivated by an
energy-weighted teacher--student alignment objective, we propose
\textbf{Graph Geometry-aware MLP (G$^2$MLP)}, a training-time
distillation framework guided by Ollivier--Ricci curvature. Curvature
acts as a local proxy for allocating supervision between
prediction-level and representation-level alignment. The deployed model is a standard MLP with no graph access at
inference. Across node-classification benchmarks,
G$^2$MLP improves over graph-free distillation baselines in most
settings, reduces the teacher--student rank gap in both regimes, and transfers
to Graph Transformer teachers and link prediction.
\end{abstract}

\section{Introduction}

Graph Neural Networks (GNNs)~\citep{gilmer2017neural} are the standard
tool for learning on graphs, but their reliance on message passing makes
inference slow in latency-sensitive applications such as fraud
detection~\citep{wang2021apan}. A growing line of work distils GNN
knowledge into a structure-free Multi-Layer Perceptron (MLP), enabling
fast graph-free inference~\citep{zhang2022graphless,guo2023linkless,tian2022learning}.
Most methods follow GLNN~\citep{zhang2022graphless} and match teacher
logits with a KL loss, sometimes injecting structural priors through
positional encodings~\citep{tian2022learning,wang2023graph},
codebooks~\citep{yang2023vqgraph}, or neighborhood alignment
cues~\citep{hu2021graph,wu2023extracting}. These methods improve the
student's information source but treat graph structure as uniformly
useful: every node receives the same distillation pressure regardless
of topology.

We argue that this uniform view misses an important source of
teacher--student mismatch. Message passing changes representations
differently across graph regions. Boundary nodes often contain
graph-induced directions that are hard to infer from features alone,
while dense interior nodes are often smoothed toward class-level
prototypes. As shown in Figure~\ref{fig:rank-mismatch}, this produces
two spectral failure modes in GNN-to-MLP distillation. On sparse graphs,
the student has lower effective rank than the teacher, which we call
\emph{spectral underfit}. On dense graphs, the teacher may collapse to a
low-rank geometry, while the feature-only student retains unsupported
variation, which we call \emph{spectral overfit}.

Our analysis in Section~\ref{sec:theory} formalizes this phenomenon
through the mismatch between the GNN kernel and the MLP kernel. The
resulting energy-weighted alignment objective $\mathcal{L}^*$ explains
why pointwise KL distillation is insufficient: two students can match
the same logits while inducing different representation geometries,
so a practical surrogate should allocate supervision where the
spectral error concentrates. Ollivier--Ricci curvature (ORC) provides
such a local signal: low-curvature edges mark
boundary regions where random-walk neighborhoods disagree, while
high-curvature edges mark dense regions where aggregation strongly
smooths representations.

We instantiate this principle in \textbf{Graph Geometry-aware MLP
(G$^2$MLP)}. G$^2$MLP uses ORC to approximate this principled alignment:
low-curvature nodes receive stronger prediction-level
supervision to correct spectral underfit, while high-curvature nodes
receive stronger representation-level neighborhood alignment to reduce
spectral overfit. The representation term uses a closed-form diagonal
Gaussian $\mathcal{W}_2$ surrogate. All graph-dependent quantities are
used only during training, so the deployed model remains a vanilla MLP.
Our contributions are:
\begin{itemize}[leftmargin=1.5em,itemsep=1pt,topsep=2pt]
    \item We identify spectral underfit and spectral overfit as two
    failure modes of GNN-to-MLP distillation, and connect them to
    boundary and interior graph regions.
    \item We derive an energy-weighted alignment objective
    $\mathcal{L}^*$ that characterizes teacher--student geometric
    mismatch and motivates curvature-dependent allocation of
    distillation supervision.
    \item We propose G$^2$MLP, a curvature-guided distillation framework
    that combines prediction- and representation-level alignment while
    preserving graph-free MLP inference.
\end{itemize}

\begin{figure}[t]
  \centering
  \includegraphics[width=\linewidth]{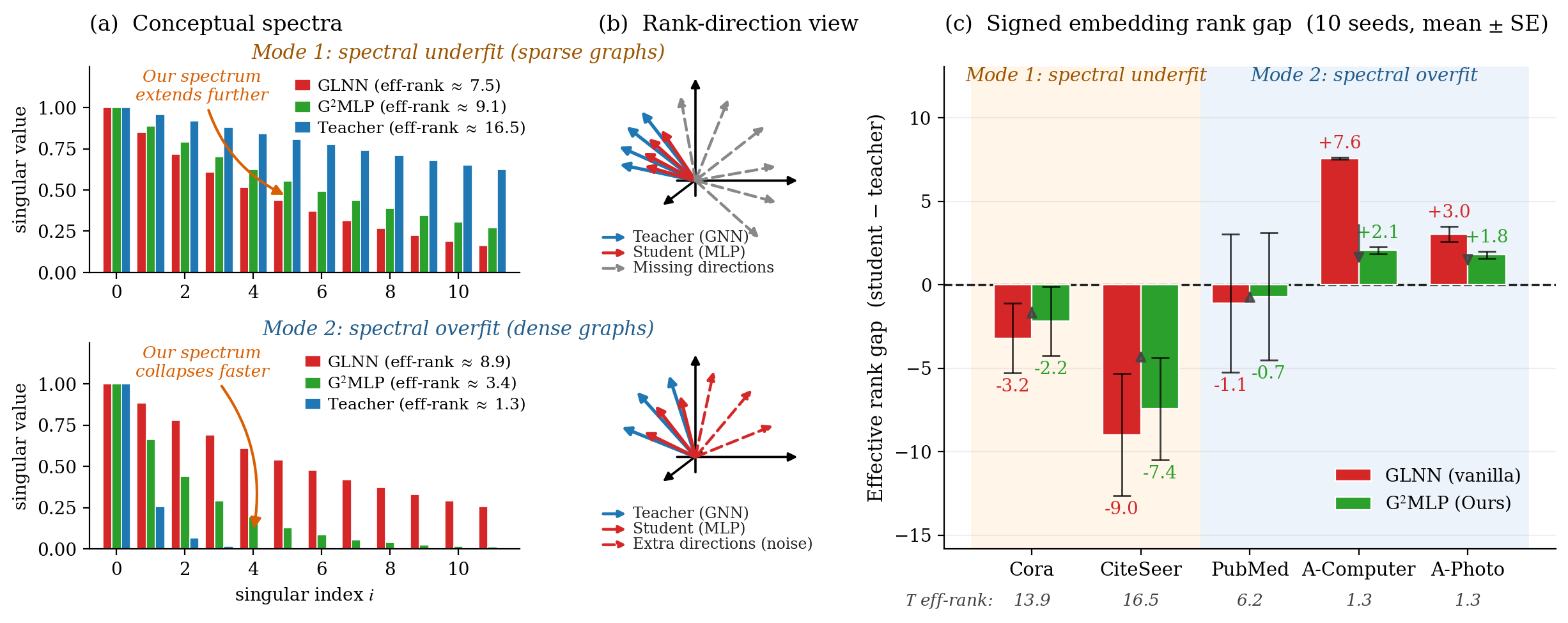}
  \caption{\textbf{Two spectral failure modes of GNN-to-MLP distillation.}
  \textbf{(a)} Singular-value spectra: GLNN exhibits \emph{spectral underfit}
  on sparse graphs and \emph{spectral overfit} on dense graphs.
  \textbf{(b)} Rank-direction sketch. \textbf{(c)} Per-dataset rank gap;
  bands mark the dominant mismatch channel (Table~\ref{tab:energy}).}
  \label{fig:rank-mismatch}
\end{figure}

\vspace{-0.4cm}
\section{Related Work}

\noindent\textbf{GNN-to-MLP knowledge distillation.}
GLNN~\citep{zhang2022graphless} established the paradigm of distilling
a GNN teacher into an MLP student through soft-label
KL~\citep{hinton2015distilling}, eliminating graph access at inference.
Follow-up work enriches the student's input or supervision target:
positional encodings (NOSMOG~\citep{tian2022learning}), structural
codebooks (VQGraph~\citep{yang2023vqgraph}), neighborhood alignment
(FF-G2M~\citep{wu2023extracting}), node-pair grouping
(MSN-GDM~\citep{bibal2024msn}), and hardness-aware reweighting
(KRD~\citep{wu2023krd}, HGMD~\citep{wu2024hgmd},
AdaGMLP~\citep{lu2024adagmlp});
SimMLP~\citep{yang2024simmlp} uses a self-supervised
alignment target. These methods either treat nodes uniformly or derive node
weights from teacher confidence; none ties reweighting to graph
topology or to a named failure mode of uniform KD.

\noindent\textbf{Optimal transport for distillation.}
WCoRD~\citep{chen2021wcord} and EMD+IPOT~\citep{lohit2022emdipot}
transfer knowledge via discrete Wasserstein distances between
batch features through cross-instance matching. WKD~\citep{chen2024wkd} introduces continuous
Gaussian $\mathcal{W}_2$ as a Riemannian metric for feature
distillation, sharper than geometry-unaware
KL~\citep{peyre2019computational}. We extend WKD's Gaussian-$\mathcal{W}_2$
template to the graph setting, fitting it per node over neighborhood
representations and weighting it by curvature.

\noindent\textbf{Curvature on graphs.}
Ollivier--Ricci curvature (ORC)~\citep{ollivier2009ricci} measures the
$\mathcal{W}_1$ distance between neighborhood walks of adjacent nodes,
and has been used for community detection~\citep{sia2019ollivier} and
information-flow analysis~\citep{topping2022understanding}.
In GNNs, \citet{nguyen2023revisiting} link positive curvature to
over-smoothing and negative curvature to over-squashing, motivating
rewiring (BORF) and dropout (CurvDrop~\citep{liu2023curvdrop}) schemes;
\citet{hevapathige2025depth} adapt depth via learnable
Bakry--Émery curvature. These works modify the GNN itself; we
instead use curvature to stratify distillation pressure.

\section{Why Uniform Distillation Fails under Curvature Heterogeneity}
\label{sec:theory}

We develop a kernel-motivated analysis of GNN-to-MLP distillation.
Our goal is not to derive an exact computable form of the GNTK, but to
identify which parts of the teacher geometry a graph-free student tends
to miss and how graph curvature can guide a tractable surrogate. The
argument proceeds in three steps: the teacher kernel contains
graph-dependent aggregation terms absent from the MLP kernel; this
mismatch appears as two opposite spectral failure modes; and ORC
provides a local graph signal for allocating supervision where these
modes concentrate. The result is an energy-weighted target objective
$\mathcal{L}^*$ and a curvature-guided approximation developed in
Section~\ref{sec:method}.

\noindent\textbf{Teacher and Student Kernel Mismatch.}
\label{sec:theory-kernels}
NTK theory~\citep{jacot2018neural} characterizes infinitely wide
networks trained by gradient flow: at convergence, the network output
lies in the function space induced by a fixed kernel, the network's
NTK. For a $K$-layer GNN with feature aggregation $\mathcal{A}$ and
activation $\sigma$, the GNTK $\mathbf{K}^{(K)}$ admits a recursion of
the schematic form
\begin{equation}
    \mathbf{K}^{(k)}(v, u)
    =
    \underbrace{\mathcal{A}\!\left[\Sigma^{(k-1)}\right]\!(v, u)}
    _{\text{aggregation over }(\mathcal{N}(v),\mathcal{N}(u))}
    +
    \dot{\Sigma}^{(k)}(v, u)\,\mathcal{A}\!\left[\mathbf{K}^{(k-1)}\right]\!(v, u),
    \label{eq:gntk-schematic}
\end{equation}
where $\Sigma^{(k)}$ is the post-activation covariance and
$\dot{\Sigma}^{(k)}$ its derivative kernel; the full recursion follows
\citet{du2019graph} and is given in Appendix~\ref{app:gntk}.
Here $\mathcal{G}$ denotes the input graph, $\mathbf{X}$ the node feature matrix, and $\mathcal{N}(v)$ the 1-hop neighborhood of $v$.
The aggregation term sums covariance over neighborhood pairs, so
$\mathbf{K}^{(K)}$ depends jointly on $(\mathbf{X},\mathcal{G})$.
The MLP student's NTK $\mathbf{K}^{\mathrm{MLP}}$ is the
aggregation-free counterpart and depends only on $\mathbf{X}$. Thus
$\Delta^{(K)}=\mathbf{K}^{(K)}-\mathbf{K}^{\mathrm{MLP}}$ captures the
graph-dependent component introduced by message passing. The following
statement is used qualitatively to motivate local surrogates; we do not
claim that the full GNTK is sparse or that curvature exactly estimates
its eigenspectrum.

\begin{lemma}[Aggregation Locality]
    \label{lem:locality}
    For a $K$-layer message-passing GNN with Lipschitz activation, the
    graph-dependent aggregation component of $\mathbf{K}^{(K)}$ is
    generated through paths contained in $K$-hop receptive fields. Hence
    the graph-induced contribution associated with a node pair $(v,u)$
    depends on the overlap and feature similarity of their $K$-hop
    neighborhoods. The full kernel need not be sparse, since feature
    covariance may remain nonzero for distant nodes.
\end{lemma}

\subsection{Two Failure Modes and the Ideal Distillation Objective}
\label{sec:theory-modes}

\noindent\textbf{Spectral distortion.}
We characterize the student's coverage of the teacher's geometry in
the teacher's eigenbasis. Let $\mathbf{K}^{(K)}=\sum_i
\lambda_i^K\mathbf{u}_i\mathbf{u}_i^\top$, with
$\lambda_1^K\geq\cdots\geq\lambda_N^K\geq0$. For each direction with
$\lambda_i^K>0$, define
\begin{equation}
    \rho_i
    =
    \frac{\mathbf{u}_i^\top \mathbf{K}^{\mathrm{MLP}}\mathbf{u}_i}
    {\lambda_i^K}.
    \label{eq:rho-def}
\end{equation}
When $\rho_i<1$, the student has insufficient energy along a teacher
direction, which we call \emph{spectral underfit}; when $\rho_i>1$, the
student retains energy in directions that the teacher kernel weakly
supports, which we call \emph{spectral overfit}. We summarize the
diagonal mismatch in the teacher eigenbasis by
\begin{equation}
    \mathcal{D}
    =
    \sum_i(\rho_i-1)^2\lambda_i^K .
    \label{eq:agg-distortion}
\end{equation}
This quantity does not compare all off-diagonal kernel entries. It
measures how much MLP energy differs from teacher energy along teacher
eigendirections, with high-energy teacher directions receiving larger
weight. Understanding where these discrepancies arise is therefore key
to allocating distillation effort.

Define Ollivier--Ricci curvature as
$\kappa(i,j)=1-W_1(\mu_i,\mu_j)$, where $\mu_v$ is the lazy random-walk
distribution on the closed neighborhood $\mathcal{N}^+(v)$. Let
$\bar{\kappa}_v$ average $\kappa(v,u)$ over $u\in\mathcal{N}(v)$, and
let $\widetilde{\kappa}_v=(\bar{\kappa}_v-\mu_\kappa)/\sigma_\kappa$.
For a Lipschitz feature map $\phi$, define the neighborhood embedding
$\bar{\mathbf{x}}_v=\sum_u\mu_v(u)\phi(\mathbf{x}_u)$. The following
proposition links ORC to local aggregation disagreement under explicit
metric assumptions.

\begin{proposition}[ORC as a Local Mismatch Proxy]
    \label{prop:mismatch}
    Suppose $\phi$ is $L$-Lipschitz with respect to the ground metric
    used in $W_1$. Then for any edge $(i,j)$,
    \begin{equation}
        \|\bar{\mathbf{x}}_i-\bar{\mathbf{x}}_j\|_2
        \leq
        L\,W_1(\mu_i,\mu_j)
        =
        L(1-\kappa(i,j)).
        \label{eq:lipschitz-agg}
    \end{equation}
    Suppose in addition that on a set of edges
    $\mathcal{E}'\subseteq\mathcal{E}$, with constants $\eta\in(0,1]$ and
    $m>0$, \textup{(C1)} the displacements transported by an optimal
    coupling $\pi^\star$ of $(\mu_i,\mu_j)$ do not cancel,
    $\|\mathbb{E}_{\pi^\star}[\phi(\mathbf{x}_X)-\phi(\mathbf{x}_Y)]\|_2
    \geq\eta\,\mathbb{E}_{\pi^\star}\|\phi(\mathbf{x}_X)-\phi(\mathbf{x}_Y)\|_2$,
    and \textup{(C2)} $\phi$ is lower-Lipschitz on the transported pairs,
    $\|\phi(\mathbf{x}_a)-\phi(\mathbf{x}_b)\|_2\geq m\,d_{\mathcal{G}}(a,b)$
    whenever $\pi^\star(a,b)>0$. Then for every $(i,j)\in\mathcal{E}'$,
    \begin{equation}
        m\eta\,(1-\kappa(i,j))
        \leq
        \|\bar{\mathbf{x}}_i-\bar{\mathbf{x}}_j\|_2
        \leq
        L\,(1-\kappa(i,j)).
        \label{eq:two-sided}
    \end{equation}
\end{proposition}

\begin{insightbox}
\textbf{Proposition~\ref{prop:mismatch}} makes ORC a locally computable proxy for where graph-dependent
teacher information is likely to matter, without requiring construction
or diagonalization of $\Delta^{(K)}$. This motivates the
curvature-stratified weights in Section~\ref{sec:curv-metric}.
\end{insightbox}

\noindent\textbf{Why pointwise KL is insufficient.}
The standard distillation loss $\mathcal{L}_{\mathrm{KD}} =
\frac{1}{|\mathcal{V}|}\sum_v D_{\mathrm{KL}}(\hat{\mathbf{y}}_v^T
\| \hat{\mathbf{y}}_v^S)$ sums node-wise divergences that depend on
the teacher's marginal predictive distribution. Two students may match
the same teacher logits while inducing different neighborhood
geometries in hidden space. Consequently, pointwise KL does not
directly control the graph-induced representation geometry encoded by
$\mathbf{K}^{(K)}$. Reducing spectral distortion therefore requires an
additional objective sensitive to neighborhood-level teacher-student
alignment.

\noindent\textbf{The target alignment loss.}
A natural energy-weighted target is
\begin{equation}
    \mathcal{L}^*
    =
    \|\mathbf{f}^T - \mathbf{f}^S\|_{\mathbf{K}^{(K)}}^2
    =
    \sum_i \lambda_i^K \!\left(\hat{f}_i^T - \hat{f}_i^S\right)^2,
    \label{eq:ideal-loss}
\end{equation}
where $\hat{f}_i^T=\mathbf{u}_i^\top\mathbf{f}^T$ and
$\|\mathbf{f}\|_{\mathbf{K}^{(K)}}^2=
\mathbf{f}^\top\mathbf{K}^{(K)}\mathbf{f}$ denotes the finite-dimensional
teacher-kernel energy form, not the RKHS norm. The weights
$\lambda_i^K$ emphasize teacher directions with high functional energy.
In spectral underfit, this prioritizes teacher directions that the MLP
does not recover; in spectral overfit, it downweights near-null teacher
directions where the student retains unsupported variation. Directly
using $\mathcal{L}^*$ is impractical because it requires storing and
diagonalizing the full teacher kernel. We therefore use it as a
principled target and approximate it with local, curvature-guided
alignment in Section~\ref{sec:method}.

\subsection{Curvature Stratifies the Failure Modes}
\label{sec:theory-curvature}

\noindent\textbf{Mode-conditional failure structure.}
Proposition~\ref{prop:mismatch} indicates that curvature separates
locally incoherent boundary neighborhoods from coherent interior
neighborhoods. Let $\mathbf{H}^T,\mathbf{H}^S\in\mathbb{R}^{N\times d_h}$
denote penultimate-layer representations, and define the demeaned
effective rank $\mathrm{erank}(\mathbf{H})=\exp(-\sum_i p_i\log p_i)$,
where $p_i\propto\sigma_i^2(\mathbf{H}-\bar{\mathbf{H}})$. 
The rank gap $\Delta r=\mathrm{erank}(\mathbf{H}^S)-\mathrm{erank}(\mathbf{H}^T)$
reduces each learned representation's spectrum to one number and ignores
its directions, so it is used as an observable diagnostic of representation
geometry, not as an exact estimator of the teacher-kernel spectrum.
Writing $\Delta_{ij}\bar{\mathbf{x}}=\bar{\mathbf{x}}_i-\bar{\mathbf{x}}_j$ for
the neighborhood contrast of an edge, the two regimes read as follows.

\begin{itemize}
    \item \textbf{Spectral underfit ($\Delta r<0$).}
    On sparse graphs with heterogeneous curvature, aggregation creates
    teacher directions that the feature-only student does not recover.
    Take $\phi$ in Proposition~\ref{prop:mismatch} to be the teacher's
    prediction map. If the student under-responds at an edge,
    $\|\Delta_{ij}\bar{\mathbf{z}}^S\|_2\leq(1-\zeta)
    \|\Delta_{ij}\bar{\mathbf{z}}^T\|_2$ with $\zeta\in(0,1]$, the
    reverse triangle inequality and Eq.~\eqref{eq:two-sided} give a
    prediction-contrast mismatch of at least $\zeta m\eta(1-\kappa(i,j))$;
    for $\zeta$, $m$ and $\eta$ common to the edges compared, this
    guarantee is largest on low-curvature edges.

    \item \textbf{Spectral overfit ($\Delta r>0$).}
    On dense graphs with relatively homogeneous curvature, aggregation
    collapses teacher representations toward low-rank class prototypes.
    Take $\phi$ to be the teacher's representation map. If the
    feature-only student retains a neighborhood representation contrast
    $c_{ij}=\|\Delta_{ij}\bar{\mathbf{h}}^S\|_2$, measured on a scale
    common with the teacher, that does not shrink with curvature, the
    upper bound in Eq.~\eqref{eq:lipschitz-agg} leaves an excess of at
    least $c_{ij}-L(1-\kappa(i,j))$, which grows toward high curvature.
\end{itemize}

\noindent\textbf{Empirical localization.}
Appendix~\ref{app:diagnostics} attributes the spectral mismatch to
individual nodes through the per-direction ratio $\rho_i$ of demeaned,
trace-normalized representations, which splits it into an underfit
part ($\rho_i<1$) and an overfit part ($\rho_i>1$), and bins nodes by
ORC quintile. For the vanilla-KL student, overfit energy is higher in
the highest than in the lowest quintile on all five benchmarks
(Figure~\ref{fig:binned}; channel totals in Table~\ref{tab:energy}),
whereas underfit energy is spread across quintiles rather than
concentrated at low curvature; at the edge level, the prediction mismatch
across lowest-curvature edges exceeds that across highest-curvature
edges on the underfit-dominated graphs. Which mode dominates is thus a
graph-level property, while curvature localizes the overfit channel
within a graph.

\begin{figure}[H]
  \centering
  \includegraphics[width=\linewidth]{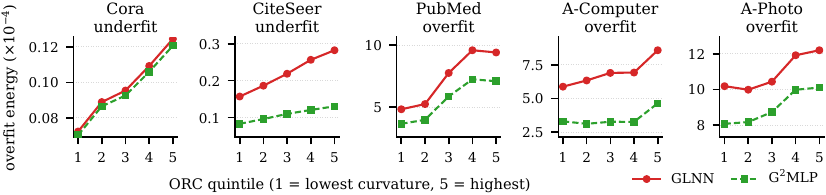}
  \caption{Mean overfit energy per ORC quintile of the nodes for the
  vanilla-KL student (GLNN) and G$^2$MLP (SAGE teacher;
  Appendix~\ref{app:diagnostics}). Every panel plots the overfit
  channel; the label under each dataset name is the graph's dominant
  mismatch channel (Table~\ref{tab:energy}).}
  \label{fig:binned}
\end{figure}

\begin{insightbox}
\textbf{Implication for the loss design.}
The two regimes call for corrections at opposite ends of the curvature
spectrum: under underfit, the guaranteed prediction mismatch is largest
on low-curvature edges; under overfit, the guaranteed excess
representation variation grows toward high curvature. A uniformly
weighted loss does not distinguish them, since pressure on
boundary-discriminative directions is also pressure on already-smoothed
interior regions. This motivates a curvature-stratified objective in
which prediction-side and representation-side alignment receive
opposite curvature-dependent emphasis, as developed in
Section~\ref{sec:method}.
\end{insightbox}

\section{G$^2$MLP: Curvature-Guided Graph Wasserstein Alignment}
\label{sec:method}

The target objective $\mathcal{L}^*$ in Eq.~\eqref{eq:ideal-loss} is
computationally intractable because it requires the full teacher kernel
spectrum. We approximate it through three practical relaxations: global
kernel alignment is replaced by local neighborhood distribution
matching, discrete optimal transport is replaced by a diagonal Gaussian
$\mathcal{W}_2$ surrogate, and uniform node weighting is replaced by
ORC-dependent weighting. These steps are tractable approximations, not exact equivalences.
Figure~\ref{fig:framework} summarizes the three steps within the
distillation pipeline.

\begin{figure}[t]
  \centering
  \includegraphics[width=\linewidth]{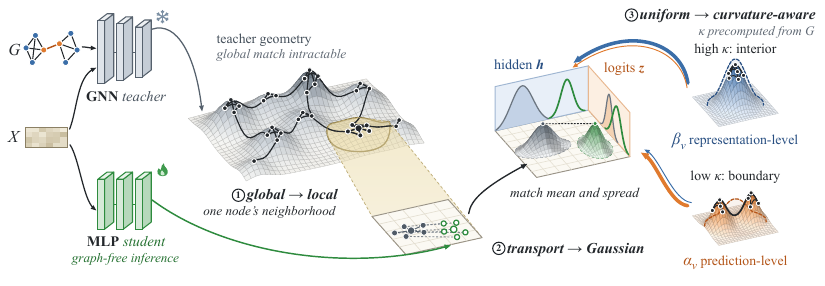}
  \caption{\textbf{Overview of G$^2$MLP.} A frozen GNN teacher is distilled
into a graph-free MLP student through three practical relaxations of
global alignment (Section~\ref{sec:method}): (1) teacher and student
representations are compared over each node's neighborhood; (2) each
neighborhood is summarized by a diagonal Gaussian; (3)
Ollivier--Ricci curvature redistributes
emphasis between the feature side ($\beta_v$, interior nodes) and the
logit side ($\alpha_v$, boundary nodes).}
  \label{fig:framework}
\end{figure}

\subsection{From Global Alignment to Local Neighborhood Measures}
\label{sec:localization}

The first source of intractability is the global scope of
$\mathbf{K}^{(K)}$. By Lemma~\ref{lem:locality}, the graph-dependent
part of the teacher kernel is generated through message-passing
neighborhoods, even though the full kernel may also contain nonlocal
feature covariance. We therefore approximate the energy-weighted
alignment in $\mathcal{L}^*$ by matching local teacher and student
representation distributions.

For each node $v$, let $\mu_v$ be the lazy random walk distribution
over $\mathcal{N}^+(v)$, defined as $\mu_v(u)=
\frac{1}{2}\mathbf{1}\{u=v\}+\frac{1}{2|\mathcal{N}(v)|}
\mathbf{1}\{u\in\mathcal{N}(v)\}$, and define the local
representation measures
\begin{equation}
    \nu_v^T
    =
    \sum_{u \in \mathcal{V}} \mu_v(u)\,\delta_{\mathbf{r}_u^T},
    \qquad
    \nu_v^S
    =
    \sum_{u \in \mathcal{V}} \mu_v(u)\,\delta_{\mathbf{r}_u^S},
    \label{eq:local-measures-method}
\end{equation}
where $\delta_{\mathbf{x}}$ denotes the Dirac measure at point
$\mathbf{x}$ and
$\mathbf{r}_u=[\mathbf{h}_u^\top,\mathbf{z}_u^\top]^\top$
concatenates the penultimate hidden features and output logits of
node $u$. Matching $\nu_v^T$ and $\nu_v^S$ in Wasserstein-2 distance
encourages agreement in both local means and local dispersion.
Aggregating over nodes gives the \emph{Graph Wasserstein objective}
\begin{equation}
    \mathcal{L}_{\mathrm{GW}}
    =
    \frac{1}{|\mathcal{V}|}
    \sum_{v \in \mathcal{V}}
    \mathcal{W}_2^2\!\left(\nu_v^T,\nu_v^S\right).
    \label{eq:gw-loss}
\end{equation}
Unlike pointwise KL, this term depends on neighborhood representation
geometry: students with identical nodewise predictions can receive
different gradients if their local representation distributions differ.
In practice, we use the uniform one-hop proxy
$q_v(u)=|\mathcal{N}^+(v)|^{-1}\mathbf{1}\{u\in\mathcal{N}^+(v)\}$ for
computational simplicity, and write $\widetilde{\nu}_v$ for the
resulting empirical measure.

\subsection{Efficient Neighborhood Alignment via Moment Matching}
\label{sec:gaussian-surrogate}

Each $\mathcal{W}_2^2(\widetilde{\nu}_v^T,\widetilde{\nu}_v^S)$ is an
optimal transport problem over $|\mathcal{N}^+(v)|$ atoms, costing
$O(|\mathcal{N}^+(v)|^3)$ per node. To obtain a scalable surrogate, we
match the first two marginal moments of the teacher and student
neighborhood measures. This approximation is most reliable in
high-curvature neighborhoods, where adjacent nodes have large
structural overlap and the local representation distribution is often
concentrated. We use the diagonal Gaussian approximation
\begin{align}
    \widetilde{\nu}_v^T
    \approx
    \mathcal{N}(\mathbf{m}_v^T,\operatorname{diag}(\boldsymbol{\sigma}_v^T)),
    \quad
    \widetilde{\nu}_v^S
    \approx
    \mathcal{N}(\mathbf{m}_v^S,\operatorname{diag}(\boldsymbol{\sigma}_v^S)),
    \label{eq:diag-gaussian-local}
\end{align}
where $\mathbf{m}_v=\sum_{u\in\mathcal{N}^+(v)}q_v(u)\mathbf{r}_u$ and
$\boldsymbol{\sigma}_v=
\sum_{u\in\mathcal{N}^+(v)}q_v(u)(\mathbf{r}_u-\mathbf{m}_v)^2$
denotes the coordinate-wise variance. The resulting closed-form
distance is
\begin{equation}
    \mathcal{W}_{2,\mathrm{diag}}^2(\widetilde{\nu}_v^T,\widetilde{\nu}_v^S)
    =
    \|\mathbf{m}_v^T-\mathbf{m}_v^S\|_2^2
    +
    \|\sqrt{\boldsymbol{\sigma}_v^T+\epsilon}
    -\sqrt{\boldsymbol{\sigma}_v^S+\epsilon}\|_2^2 .
    \label{eq:diag-gaussian-w2}
\end{equation}
This reduces the per-node cost to $O(d)$. The diagonal covariance
follows standard practice in Gaussian-$\mathcal{W}_2$
distillation~\citep{chen2024wkd}; it avoids the $O(d^2)$ cost of full
covariances while retaining coordinate-wise dispersion. The surrogate
may be less accurate for multimodal or strongly heterophilic
neighborhoods, which we treat as a limitation. Since
$\mathbf{r}_u=[\mathbf{h}_u^\top,\mathbf{z}_u^\top]^\top$, the distance
in Eq.~\eqref{eq:diag-gaussian-w2} splits additively into a feature-side
term over $\mathbf{h}$ and a logit-side term over $\mathbf{z}$,
enabling the curvature-dependent weighting in
Section~\ref{sec:curv-metric}. We denote this objective
$\mathcal{L}_{\mathrm{diag}}$.

\subsection{Curvature-Adaptive Weighting}
\label{sec:curv-metric}

The remaining gap between $\mathcal{L}_{\mathrm{diag}}$ and
$\mathcal{L}^*$ is that $\mathcal{L}_{\mathrm{diag}}$ weights every node
uniformly, while $\mathcal{L}^*$ emphasizes high-energy teacher
directions. Section~\ref{sec:theory-curvature} suggests that these
directions are associated with different curvature regimes in the two
failure modes: low-curvature boundary regions are important for
recovering missing decision-relevant teacher directions, whereas
high-curvature interior regions are important for suppressing
unsupported student variation after teacher smoothing. We therefore use
ORC as a local proxy for allocating the alignment budget, rather than as
an exact estimator of kernel eigenvalues.

We instantiate this proxy with opposite monotone weights for the logit
and feature components:
\begin{equation}
    \alpha_v
    =
    \frac{\exp(-\gamma_z\widetilde{\kappa}_v)}{
    \frac{1}{|\mathcal{V}|}\sum_{u}\exp(-\gamma_z\widetilde{\kappa}_u)},
    \qquad
    \beta_v
    =
    \frac{\exp(+\gamma_h\widetilde{\kappa}_v)}{
    \frac{1}{|\mathcal{V}|}\sum_{u}\exp(+\gamma_h\widetilde{\kappa}_u)},
    \label{eq:curv-weights-method}
\end{equation}
where $\alpha_v$ upweights prediction-side alignment at low-curvature
boundary nodes and $\beta_v$ upweights representation-side alignment at
high-curvature interior nodes. The normalization enforces
$\frac{1}{|\mathcal{V}|}\sum_v\alpha_v=
\frac{1}{|\mathcal{V}|}\sum_v\beta_v=1$, so curvature redistributes
emphasis without changing the overall loss scale. Hyperparameters
$\gamma_z,\gamma_h>0$ control the strength of stratification.

Let $\widetilde{\nu}_{h,v}$ and $\widetilde{\nu}_{z,v}$ denote the
hidden and logit marginals of $\widetilde{\nu}_v$. The resulting
\emph{Curvature-Adaptive Wasserstein} objective is
\begin{equation}
    \mathcal{L}_{\mathrm{CAW}}
    =
    \frac{1}{|\mathcal{V}|}
    \sum_{v \in \mathcal{V}}
    \!\left[
    \beta_v\,
    \mathcal{W}_{2,\mathrm{diag}}^2\!\!
    \left(\widetilde{\nu}_{h,v}^T,\widetilde{\nu}_{h,v}^S\right)
    +
    \alpha_v\,
    \mathcal{W}_{2,\mathrm{diag}}^2\!\!
    \left(\widetilde{\nu}_{z,v}^T,\widetilde{\nu}_{z,v}^S\right)
    \right],
    \label{eq:caw-loss}
\end{equation}
which is computable in $O((|\mathcal{V}| + |\mathcal{E}|) \cdot d)$ per iteration.
\begin{insightbox}
\textbf{The objective in Eq.~(\ref{eq:caw-loss})} keeps the approximation aligned with the two regimes:
low-curvature nodes receive stronger prediction-level correction, while
high-curvature nodes receive stronger neighborhood representation
matching where the Gaussian moment surrogate is most appropriate.
\end{insightbox}

\noindent\textbf{Objective Function.}
\label{sec:final-objective}
The full G$^2$MLP training objective is
\begin{equation}    \mathcal{L}
    =
    \lambda_{\mathrm{CE}}\mathcal{L}_{\mathrm{CE}}
    +
    \lambda_{\mathrm{KD}}\mathcal{L}_{\mathrm{KD}}
    +
    \lambda_{\mathrm{CAW}}
    \frac{\mathcal{L}_{\mathrm{CAW}}}{
    \operatorname{sg}(\mathcal{L}_{\mathrm{CAW}})+\epsilon_{\mathrm{sg}}},
    \label{eq:final-g2mlp}
\end{equation}
where $\mathcal{L}_{\mathrm{KD}}=
\frac{1}{|\mathcal{V}|}\sum_v
D_{\mathrm{KL}}(\hat{\mathbf{y}}_v^T\|\hat{\mathbf{y}}_v^S)$ and
$\mathcal{L}_{\mathrm{CE}}$ is cross-entropy on labeled nodes
$\mathcal{V}^L$, with $\lambda_{\mathrm{CE}}=0$ in all
node-classification and Graph Transformer experiments. The stop-gradient normalization rescales the CAW
gradient by a batch-level constant without changing its direction,
making its magnitude comparable to entropy-scaled losses. The KL term
preserves probability-simplex calibration, while $\mathcal{L}_{\mathrm{CAW}}$
aligns local teacher-student geometry.

\section{Experiments}
\label{sec:experiments}
\subsection{Experimental Setup}
\label{sec:setup}
\noindent\textbf{Datasets.}
We evaluate on six node classification benchmarks: three citation
networks (Cora, Citeseer, Pubmed), two Amazon co-purchase graphs
(A-Computer, A-Photo)~\citep{shchur2018pitfalls}, and the large-scale
OGB graph ogbn-Arxiv~\citep{hu2020open}. The first five follow the
CPF setting~\citep{yang2021extract}; ogbn-Arxiv follows the standard
OGB split. Dataset statistics are given in Appendix~\ref{app:exp-details}.

\noindent\textbf{Settings.}
We follow the protocol of GLNN~\citep{zhang2022graphless}, evaluating the result in two regimes:
\emph{transductive}, where all nodes are visible during training but
test labels are hidden; and
\emph{inductive}, where a disjoint set of test nodes (and their
incident edges) is held out from training. The interpolated
\emph{production} score $\mathrm{prod} = 0.2\!\times\!\mathrm{ind} +
0.8\!\times\!\mathrm{tran}$~\citep{tian2022learning} summarizes a
realistic deployment mix of seen and unseen nodes.

\noindent\textbf{Baselines.}
We focus on the canonical GNN-to-MLP distillation setup, defined by two
conditions: (a) the deployed student consumes \textbf{only the original
node features} $\mathbf{x}_v$, and (b) the teacher is a
\textbf{fixed, independently trained GNN}; training-time graph access
is not a criterion, since every distillation method uses the graph
through its teacher. The GNN teacher is
GraphSAGE~\citep{hamilton2017inductive} throughout, following GLNN,
NOSMOG. We compare against: vanilla MLP; GLNN~\citep{zhang2022graphless},
soft-label KL distillation; KRD~\citep{wu2023krd}, reliability-based
sample reweighting; FF-G2M~\citep{wu2023extracting}, low- and
high-frequency distillation. NOSMOG~\citep{tian2022learning}, whose
student consumes graph-derived positional encodings at inference,
VQGraph~\citep{yang2023vqgraph}, which trains the teacher jointly with
a codebook, HGMD~\citep{wu2024hgmd}, a hardness-aware reweighting
scheme, and AdaGMLP~\citep{lu2024adagmlp}, a boosted ensemble of MLP
students, are compared separately in Section~\ref{sec:beyond}.

\noindent\textbf{Implementation.}
We use a 2-layer GraphSAGE teacher and a 2-layer MLP student with
matched hidden dimension, following standard practice~\citep{zhang2022graphless};
ogbn-Arxiv uses the OGB-standard 3-layer architecture for both teacher and student.
All models are trained with Adam~\citep{kingma2014adam} and early
stopping on validation accuracy. Student
learning rate, weight decay, and dropout follow the GLNN-released
configurations and are not tuned; our loss-balance hyperparameters are
set per dataset.
Hyperparameter ranges and per-dataset optimizer settings are in
Appendix~\ref{app:exp-details}. We report mean accuracy $\pm$ standard
deviation over 10 unique random seeds.

\subsection{Main Results}
\vspace{-1em}
\begin{table}[H]
\caption{Node classification accuracy (\%), \textbf{transductive} setting.
\textbf{Bold}: best graph-free student.}
\label{tab:transductive}
\centering
\small
\resizebox{\textwidth}{!}{%
\begin{tabular}{@{}l l |llll l@{}}
\toprule
& Teacher & \multicolumn{5}{c}{Graph-free students} \\
\cmidrule(lr){2-7}
Dataset & SAGE & MLP & GLNN & KRD & FF-G2M & \textbf{G$^2$MLP} \\
\midrule
Citeseer &70.49$\pm$1.53 & 58.50$\pm$1.86  & 71.22$\pm$1.50  & 72.84$\pm$1.70  & 72.85$\pm$1.59  & \textbf{74.51$\pm$2.35} \\
Pubmed &75.56$\pm$2.06 & 68.39$\pm$3.09  & 75.59$\pm$2.46  & 77.01$\pm$3.11  & 76.56$\pm$3.41  & \textbf{78.17$\pm$2.75} \\
Cora &80.64$\pm$1.57 & 59.18$\pm$1.60  & 80.26$\pm$1.66  & 82.27$\pm$1.31  & 82.38$\pm$1.41  & \textbf{82.54$\pm$1.87} \\
A-computer &82.82$\pm$1.37 & 67.62$\pm$2.21  & 82.71$\pm$1.18  & 82.87$\pm$0.87  & 83.67$\pm$1.04  & \textbf{83.87$\pm$1.09} \\
A-photo &90.85$\pm$0.87 & 77.29$\pm$1.79  & 91.95$\pm$1.04  & 91.95$\pm$1.44  & 92.51$\pm$0.56  & \textbf{92.85$\pm$0.78} \\
Arxiv &70.73$\pm$0.35 & 55.67$\pm$0.24  & 63.75$\pm$0.48  & 59.26$\pm$0.51  & 58.51$\pm$0.35  & \textbf{64.57$\pm$0.31} \\
\bottomrule
\end{tabular}
}
\vspace{-1em}
\end{table}

\noindent\textbf{Transductive Setting.}
\label{sec:transductive}
G$^2$MLP achieves the best graph-free accuracy on all six datasets,
improving over the strongest baseline (FF-G2M, KRD, or GLNN depending
on the dataset) by 0.16 to 1.66 pp. G$^2$MLP further surpasses the
GraphSAGE teacher on five of the six datasets, with gains up to
4.02 pp on Citeseer; only on ogbn-Arxiv does the student fall short
of the teacher (-6.16 pp), as feature-only inference is known to be
hard on this graph~\citep{zhang2022graphless}.
Component ablation (Appendix~\ref{app:ablation}) confirms that gains arise from curvature stratification.
\vspace{-1em}

\begin{table}[h]
\caption{Node classification accuracy (\%) under the \textbf{production} setting; conventions as in Table~\ref{tab:transductive}.}
\label{tab:inductive}
\centering
\small
\resizebox{\textwidth}{!}{%
\begin{tabular}{@{}ll l | llll l@{}}
\toprule
& & Teacher & \multicolumn{5}{c}{Graph-free students} \\
\cmidrule(lr){3-8}
Dataset & Eval & SAGE & MLP & GLNN & KRD & FF-G2M & \textbf{G$^2$MLP} \\
\midrule
\multirow{3}{*}{Citeseer}
 & prod &68.06 & 58.49  & 69.08  & 71.38  & \textbf{71.89}  & 71.67 \\
 & ind &69.14$\pm$2.99 & 59.31$\pm$4.56  & 68.48$\pm$2.38  & 69.78$\pm$3.04  & 69.75$\pm$3.16  & \textbf{71.85$\pm$2.27} \\
 & tran &67.79$\pm$2.80 & 58.29$\pm$1.94  & 69.23$\pm$2.39  & 71.77$\pm$2.81  & \textbf{72.12$\pm$2.69}  & 71.62$\pm$1.31 \\
\midrule
\multirow{3}{*}{Pubmed}
 & prod &74.77 & 68.39  & 74.67  & 76.00  & 73.98  & \textbf{77.45} \\
 & ind &75.07$\pm$2.89 & 68.28$\pm$3.25  & 74.52$\pm$2.95  & 75.17$\pm$3.11  & 73.49$\pm$7.91  & \textbf{77.30$\pm$1.48} \\
 & tran &74.70$\pm$2.33 & 68.42$\pm$3.06  & 74.70$\pm$2.75  & 76.20$\pm$3.00  & 74.10$\pm$7.78  & \textbf{77.49$\pm$1.25} \\
\midrule
\multirow{3}{*}{Cora}
 & prod &79.53 & 59.18  & 77.82  & 75.74  & 78.60  & \textbf{79.73} \\
 & ind &81.03$\pm$1.71 & 59.44$\pm$3.36  & 73.21$\pm$1.50  & 70.26$\pm$1.94  & 72.02$\pm$1.43  & \textbf{75.90$\pm$1.87} \\
 & tran &79.16$\pm$1.60 & 59.12$\pm$1.49  & 78.97$\pm$1.56  & 77.11$\pm$1.44  & 80.01$\pm$1.41  & \textbf{80.69$\pm$0.73} \\
\midrule
\multirow{3}{*}{A-computer}
 & prod &82.73 & 67.62  & 82.10  & 81.17  & 82.69  & \textbf{84.29} \\
 & ind &82.83$\pm$1.51 & 67.69$\pm$2.20  & 80.27$\pm$2.11  & 79.15$\pm$1.82  & 80.52$\pm$1.56  & \textbf{82.67$\pm$0.93} \\
 & tran &82.70$\pm$1.34 & 67.60$\pm$2.23  & 82.56$\pm$1.80  & 81.67$\pm$1.92  & 83.23$\pm$1.36  & \textbf{84.69$\pm$0.85} \\
\midrule
\multirow{3}{*}{A-photo}
 & prod &90.45 & 77.29  & 91.34  & 91.84  & 92.35  & \textbf{92.51} \\
 & ind &90.56$\pm$1.47 & 77.44$\pm$1.50  & 89.50$\pm$1.12  & 90.04$\pm$1.12  & \textbf{90.70$\pm$0.76}  & 90.61$\pm$0.50 \\
 & tran &90.42$\pm$0.68 & 77.25$\pm$1.90  & 91.80$\pm$0.49  & 92.29$\pm$0.63  & 92.77$\pm$0.24  & \textbf{92.98$\pm$0.51} \\
\midrule
\multirow{3}{*}{Arxiv}
 & prod &70.69 & 55.35  & 63.50  & 59.32  & 59.60  & \textbf{63.90} \\
 & ind &70.69$\pm$0.58 & 55.29$\pm$0.63  & 59.04$\pm$0.46  & 57.32$\pm$0.31  & 57.02$\pm$0.43  & \textbf{59.82$\pm$0.42} \\
 & tran &70.69$\pm$0.39 & 55.36$\pm$0.34  & 64.61$\pm$0.15  & 59.82$\pm$0.27  & 60.24$\pm$0.23  & \textbf{64.92$\pm$0.26} \\
\bottomrule
\end{tabular}%
}
\end{table}

\noindent\textbf{Inductive and Production Settings.}
\label{sec:inductive}
G$^2$MLP achieves the best graph-free production accuracy on 5/6
datasets, with gains of 0.16--1.60 pp over the strongest baseline.
The advantage is largest in the inductive regime, where G$^2$MLP
improves by 2.07--2.69 pp on Cora, Citeseer, Pubmed, and A-computer,
more than under transductive evaluation, so curvature-stratified
alignment transfers to unseen nodes. On A-computer (ind), G$^2$MLP
reaches 82.67\%, essentially matching the SAGE teacher (82.83\%)
without graph access; on Citeseer and Pubmed (prod), G$^2$MLP
surpasses the teacher by 3.61 and 2.68 pp respectively.

\subsection{Comparison Beyond the Fixed-Teacher Protocol}
\label{sec:beyond}
Table~\ref{tab:beyond} in Appendix~\ref{app:beyond} covers the methods
that relax the fixed-teacher protocol. Run with our teachers, NOSMOG
sits within $\pm$0.6\,pp of G$^2$MLP on the CPF graphs, although our
loss uses neither its representation-distillation nor its adversarial
component. On ogbn-Arxiv NOSMOG's positional input is decisive: 69.72
under our teachers (71.65 published) against our 64.57, while our
student with the same DeepWalk input reaches 71.52$\pm$0.30. G$^2$MLP
is comparable to or higher than VQGraph as reproduced by
TINED~\citep{zhou2025tined}, whose codebook-trained teacher makes this
a different-teacher comparison. Under our teachers and student width
it exceeds HGMD-M and AdaGMLP on all five CPF benchmarks; on ogbn-Arxiv
AdaGMLP's $K$-student ensemble reaches 65.51 against ours, at $K$
times the inference cost of our student.

\subsection{GT-to-MLP Distillation}
\label{sec:gt2mlp}
Graph Transformers (GTs) incur quadratic attention cost at inference,
making distillation to MLP students particularly attractive. We
evaluate three GT teachers: GT~\citep{dwivedi2020generalization}
with local masked attention, GraphGPS~\citep{rampavsek2022recipe}
with hybrid MPNN+global attention, and NAGphormer~\citep{chen2023nagphormer}
with per-node hop-token attention. Experiments use the same five
datasets as Section~\ref{sec:transductive} except ogbn-Arxiv, which
exceeds GraphGPS dense-attention memory. The student is the same
MLP used throughout.

\begin{table*}[h]
\caption{Transformer-teacher to MLP distillation. $T$: teacher;
$S$: GLNN student; $S^{\dagger}$: G$^2$MLP student. \textbf{Bold}:
best graph-free student per teacher.\\}
\label{tab:gt2mlp}
\centering
\small
\setlength{\tabcolsep}{4pt}
\resizebox{\textwidth}{!}{%
\begin{tabular}{@{}l lll lll lll@{}}
\toprule
& \multicolumn{3}{c}{GT} & \multicolumn{3}{c}{GraphGPS} & \multicolumn{3}{c}{NAGphormer} \\
\cmidrule(lr){2-4} \cmidrule(lr){5-7} \cmidrule(lr){8-10}
Dataset & $T$ & $S$ & $S^{\dagger}$ & $T$ & $S$ & $S^{\dagger}$ & $T$ & $S$ & $S^{\dagger}$ \\
\midrule
Citeseer &71.16$\pm$2.31 & 71.85$\pm$2.45  & \textbf{72.35$\pm$1.85}  &68.07$\pm$2.62 & 69.38$\pm$2.30  & \textbf{69.68$\pm$2.05}  &70.09$\pm$1.81 & 71.24$\pm$1.64  & \textbf{71.82$\pm$1.29} \\
Pubmed &74.09$\pm$2.48 & 74.61$\pm$2.48  & \textbf{74.76$\pm$2.83}  &71.77$\pm$2.81 & 71.68$\pm$2.61  & \textbf{72.03$\pm$2.58}  &76.19$\pm$2.86 & 76.85$\pm$3.34  & \textbf{77.06$\pm$3.49} \\
Cora &78.53$\pm$2.08 & 78.98$\pm$1.90  & \textbf{79.21$\pm$2.09}  &76.21$\pm$1.91 & 76.06$\pm$1.99  & \textbf{76.77$\pm$1.97}  &80.88$\pm$1.77 & 81.00$\pm$1.45  & \textbf{81.48$\pm$1.53} \\
A-computer &83.89$\pm$1.59 & 84.53$\pm$1.42  & \textbf{84.72$\pm$1.32}  &78.30$\pm$1.70 & 78.99$\pm$1.73  & \textbf{80.12$\pm$1.55}  &82.78$\pm$1.92 & 83.02$\pm$1.87  & \textbf{83.31$\pm$1.77} \\
A-photo &90.94$\pm$0.87 & 91.90$\pm$0.56  & \textbf{92.06$\pm$0.82}  &90.41$\pm$1.10 & 91.03$\pm$0.95  & \textbf{91.32$\pm$0.93}  &91.45$\pm$0.82 & 92.07$\pm$0.74  & \textbf{92.49$\pm$0.65} \\
\bottomrule
\end{tabular}%
}
\end{table*}

\noindent\textbf{Results.}
The MLP student meets or exceeds its transformer teacher on $13/15$
cells ($S \geq T$: GT $5/5$, GraphGPS $3/5$, NAGphormer $5/5$), and
G$^2$MLP improves over the vanilla-KL student on every cell
($S^{\dagger} > S$ on $15/15$) at the same student architecture and
inference cost. Teacher configurations, the
$S \geq T$ reversal, and the GraphGPS dense-vs-sparse comparison are in
Appendix~\ref{app:gt-experiments}.
GCN and GAT teachers (Appendix~\ref{app:other-teachers}) extend the
evaluation to six teacher architectures: G$^2$MLP improves over the
vanilla-KL student on Cora and Citeseer under GCN and on Cora and
Pubmed under GAT, and the remaining cells are within noise.

\subsection{Extension to Link Prediction}
\label{sec:lp}
\emph{Link prediction} (LP) tests whether the framework generalizes
beyond label fitting: supervision lives on edges and the student must
produce embeddings whose inner product reproduces the teacher's edge
ranking. The teacher is a 2-layer GraphSAGE trained with binary
cross-entropy on training edges, and the student is the same 2-layer
MLP as in node classification. We score edges by
$s(u,v) = \langle \ve_u, \ve_v \rangle$ and train the student with an
embedding-MSE distillation term and an edge-BCE term (the
\emph{multi-task} regime; Appendix~\ref{app:lp-ablation} gives the
protocol and ablates either term). Table~\ref{tab:lp-multi} reports
mean$\pm$std of test AUC / AP over 10 seeds.

\begin{table}[h]
\caption{Link prediction (multi-task regime); conventions as in Table~\ref{tab:transductive}.}
\label{tab:lp-multi}
\centering
\small
\resizebox{\textwidth}{!}{%
\begin{tabular}{@{}l ll ll ll@{}}
\toprule
& \multicolumn{2}{c}{Cora} & \multicolumn{2}{c}{Citeseer} & \multicolumn{2}{c}{Pubmed} \\
\cmidrule(lr){2-3} \cmidrule(lr){4-5} \cmidrule(lr){6-7}
Method & AUC & AP & AUC & AP & AUC & AP \\
\midrule
SAGE (teacher, ref) & 81.29$\pm$1.99 & 80.54$\pm$2.69 & 81.21$\pm$3.15 & 81.80$\pm$3.87 & 88.36$\pm$1.65 & 86.83$\pm$1.54 \\
\midrule
MLP                 & 77.92$\pm$1.78 & 76.23$\pm$2.01 & 76.47$\pm$2.08 & 76.13$\pm$2.74 & 90.29$\pm$0.38 & 89.34$\pm$0.44 \\
GLNN                & 78.61$\pm$1.68 & 77.05$\pm$1.76 & 83.32$\pm$2.34 & 83.71$\pm$2.53 & 90.03$\pm$0.38 & 89.12$\pm$0.32 \\
KRD                 & 79.20$\pm$1.65 & 77.89$\pm$2.31 & 83.31$\pm$2.23 & 83.81$\pm$2.42 & 90.32$\pm$0.42 & 89.38$\pm$0.51 \\
\textbf{G$^2$MLP (Ours)} & \textbf{83.90$\pm$1.10} & \textbf{83.36$\pm$0.99} & \textbf{83.62$\pm$2.33} & \textbf{84.32$\pm$2.43} & \textbf{90.36$\pm$1.29} & \textbf{89.96$\pm$1.07} \\
\bottomrule
\end{tabular}%
}
\end{table}

\noindent\textbf{Result.}
G$^2$MLP wins on all three datasets in both AUC and AP, with the
largest gain on Cora (+4.70 AUC over KRD, +5.47 AP) and narrow margins
on Citeseer (+0.30 / +0.51) and Pubmed (+0.04 / +0.58). All three G$^2$MLP rows exceed the SAGE teacher (Cora +2.61,
Citeseer +2.41, Pubmed +2.00 AUC); on Cora and Citeseer, where the
plain MLP stays below the teacher, the graph-free student thus
preserves and extends the teacher's edge-ranking knowledge. The plain
MLP is competitive only where node features alone carry the LP signal
(Pubmed); on Citeseer it loses 7 AUC to GLNN, so the
embedding-distillation term contributes structural information the
edge-BCE term alone cannot recover. On
Citeseer and Pubmed, where KRD already exceeds the teacher, the margin
over KRD is small but positive in both AUC and AP; all methods share
the same splits and teacher checkpoints, so the comparison is paired.

\subsection{Inference Efficiency and Scale}
\label{sec:efficiency}
\begin{wrapfigure}{r}{0.47\textwidth}
\centering
\includegraphics[width=\linewidth]{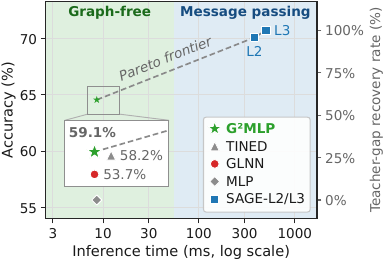}
\captionsetup{justification=raggedright,singlelinecheck=false}
\caption{Accuracy vs.\ inference time on \mbox{ogbn-Arxiv}, with the
teacher-gap recovery rate (right axis). Inset: distilled students
(8.6--9.7\,ms). L2/L3: 2-/3-layer SAGE.}
\label{fig:efficiency}
\vspace{-1.2em}
\end{wrapfigure}

G$^2$MLP shares the same MLP architecture as GLNN: curvature
reweighting is applied entirely at training time, so deployment
reduces to a feature-vector $\to$ logit forward pass with no graph
access. As Figure~\ref{fig:efficiency} shows, this yields a $58\times$ speedup
over the 3-layer GraphSAGE teacher under standard layer-wise full
inference (8.6 vs.\ 499.7 ms). In accuracy, G$^2$MLP recovers
$59.1\%$ of the teacher's gain over a vanilla MLP, ahead of GLNN
($53.7\%$, same student) and matching the non-canonical
TINED~\citep{zhou2025tined} ($58.2\%$, deeper student), placing it on
the speed--accuracy Pareto frontier.

\noindent\textbf{Scale and precomputation cost.}
Exact Ollivier--Ricci curvature requires an optimal-transport solve
per edge, $O(|\mathcal{E}|\cdot\bar{d}^3)$ for average degree
$\bar{d}$, computed once per graph. On ogbn-Arxiv and ogbn-Products we
use an $O(1)$-per-edge common-neighbor approximation
(Appendix~\ref{app:curvabl}); on ogbn-Products the weights build in 276
seconds on one CPU node and the objective trains end to end on one GPU
from a per-epoch node and edge subsample. Over ten seeds on the public
OGB split, G$^2$MLP reaches 70.03$\pm$0.30 there, above the best
published graph-free result, TINED's 69.48$\pm$0.25, each number being
its method's own best configuration.
Appendix~\ref{app:curvabl} compares
the approximation and three Forman-type
curvatures~\citep{sreejith2016forman,topping2022understanding} with
exact ORC under one recipe: all agree within 0.3\,pp on average, so
surrogates can replace exact ORC on large graphs.

\section{Conclusion}

We presented Graph Geometry-aware MLP (G$^2$MLP), a distillation
framework that targets the two spectral failure modes of uniform KD,
spectral underfit and spectral overfit, by combining
curvature-stratified reweighting with graph Wasserstein alignment at
training time only. On node
classification, G$^2$MLP reduces the embedding rank gap in
both regimes and improves over graph-free baselines on most benchmarks
at the inference cost of a vanilla MLP. The same recipe
transfers to Graph Transformer teachers and to link prediction.
Future work will extend curvature-stratified principle to dynamic graphs,
whose curvature evolves, and to heterogeneous graphs.

\section*{Acknowledgments}

The work was supported by the Professional Development (PD) Funding account (R.CE201.PD00) and University Research Scholarships, with computational resources provided by the ANU School of Computing. We thank the reviewers and area chair for their constructive feedback.
We dedicate this work to the memory of the late Professor {\setlength{\fboxsep}{1.5pt}\setlength{\fboxrule}{0.4pt}\fbox{Qing Wang}} (Australian National University), who was not only a mentor and teacher, but also a guiding light and a motherly presence. She lit the torch illuminating the first author's research; though no longer with us, her light forever guides the way.

\clearpage
\bibliographystyle{plainnat}
\bibliography{references}

\begin{thebibliography}{41}
\providecommand{\natexlab}[1]{#1}
\providecommand{\url}[1]{\texttt{#1}}
\expandafter\ifx\csname urlstyle\endcsname\relax
  \providecommand{\doi}[1]{doi: #1}\else
  \providecommand{\doi}{doi: \begingroup \urlstyle{rm}\Url}\fi

\bibitem[Chen et~al.(2023)Chen, Gao, Li, and He]{chen2023nagphormer}
Jinsong Chen, Kaiyuan Gao, Gaichao Li, and Kun He.
\newblock {NAGphormer}: A tokenized graph transformer for node classification in large graphs.
\newblock In \emph{International Conference on Learning Representations (ICLR)}, 2023.

\bibitem[Chen et~al.(2021)Chen, Wang, Gan, Liu, Henao, and Carin]{chen2021wcord}
Liqun Chen, Dong Wang, Zhe Gan, Jingjing Liu, Ricardo Henao, and Lawrence Carin.
\newblock Wasserstein contrastive representation distillation.
\newblock In \emph{CVPR}, pages 16296--16305, 2021.

\bibitem[Cho and Saul(2009)]{cho2009kernel}
Youngmin Cho and Lawrence Saul.
\newblock Kernel methods for deep learning.
\newblock \emph{Advances in neural information processing systems}, 22, 2009.

\bibitem[Dong et~al.(2025)Dong, Li, Zheng, and Kong]{bibal2024msn}
Rui Dong, Jiaxing Li, Weihuang Zheng, and Youyong Kong.
\newblock Suit the node pair to the case: A multi-scale node pair grouping strategy for graph-mlp distillation.
\newblock In \emph{Proceedings of the Thirty-Fourth International Joint Conference on Artificial Intelligence}, pages 2775--2783, 2025.

\bibitem[Du et~al.(2019)Du, Hou, Salakhutdinov, Poczos, Wang, and Xu]{du2019graph}
Simon~S Du, Kangcheng Hou, Russ~R Salakhutdinov, Barnabas Poczos, Ruosong Wang, and Keyulu Xu.
\newblock Graph neural tangent kernel: Fusing graph neural networks with graph kernels.
\newblock \emph{Advances in neural information processing systems}, 32, 2019.

\bibitem[Dwivedi and Bresson(2020)]{dwivedi2020generalization}
Vijay~Prakash Dwivedi and Xavier Bresson.
\newblock A generalization of transformer networks to graphs.
\newblock \emph{CoRR}, abs/2012.09699, 2020.
\newblock URL \url{https://arxiv.org/abs/2012.09699}.

\bibitem[Furlanello et~al.(2018)Furlanello, Lipton, Tschannen, Itti, and Anandkumar]{furlanello2018born}
Tommaso Furlanello, Zachary Lipton, Michael Tschannen, Laurent Itti, and Anima Anandkumar.
\newblock Born again neural networks.
\newblock In \emph{International conference on machine learning}, pages 1607--1616. PMLR, 2018.

\bibitem[Gilmer et~al.(2017)Gilmer, Schoenholz, Riley, Vinyals, and Dahl]{gilmer2017neural}
Justin Gilmer, Samuel~S Schoenholz, Patrick~F Riley, Oriol Vinyals, and George~E Dahl.
\newblock Neural message passing for quantum chemistry.
\newblock In \emph{ICML}, 2017.

\bibitem[Guo et~al.(2023)Guo, Shiao, Zhang, Liu, Chawla, Shah, and Zhao]{guo2023linkless}
Zhichun Guo, William Shiao, Shichang Zhang, Yozen Liu, Nitesh~V Chawla, Neil Shah, and Tong Zhao.
\newblock Linkless link prediction via relational distillation.
\newblock In \emph{ICML}, 2023.

\bibitem[Hamilton et~al.(2017)Hamilton, Ying, and Leskovec]{hamilton2017inductive}
Will Hamilton, Zhitao Ying, and Jure Leskovec.
\newblock Inductive representation learning on large graphs.
\newblock In \emph{NeurIPS}, 2017.

\bibitem[Hevapathige et~al.(2025)Hevapathige, Zehmakan, and Wang]{hevapathige2025depth}
Asela Hevapathige, Ahad~N Zehmakan, and Qing Wang.
\newblock Depth-adaptive graph neural networks via learnable {Bakry--\'Emery} curvature.
\newblock In \emph{KDD}, 2025.

\bibitem[Hinton et~al.(2015)Hinton, Vinyals, and Dean]{hinton2015distilling}
Geoffrey Hinton, Oriol Vinyals, and Jeff Dean.
\newblock Distilling the knowledge in a neural network.
\newblock \emph{arXiv:1503.02531}, 2015.

\bibitem[Hu et~al.(2020)Hu, Fey, Zitnik, Dong, Ren, Liu, Catasta, and Leskovec]{hu2020open}
Weihua Hu, Matthias Fey, Marinka Zitnik, Yuxiao Dong, Hongyu Ren, Bowen Liu, Michele Catasta, and Jure Leskovec.
\newblock Open graph benchmark: Datasets for machine learning on graphs.
\newblock In \emph{NeurIPS}, 2020.

\bibitem[Hu et~al.(2021)Hu, You, Wang, Wang, Zhou, and Gao]{hu2021graph}
Yang Hu, Haoxuan You, Zhecan Wang, Zhicheng Wang, Erjin Zhou, and Yue Gao.
\newblock {Graph-MLP}: Node classification without message passing in graph.
\newblock \emph{arXiv:2106.04051}, 2021.

\bibitem[Jacot et~al.(2018)Jacot, Gabriel, and Hongler]{jacot2018neural}
Arthur Jacot, Franck Gabriel, and Cl{\'e}ment Hongler.
\newblock Neural tangent kernel: Convergence and generalization in neural networks.
\newblock \emph{Advances in neural information processing systems}, 31, 2018.

\bibitem[Kingma and Ba(2015)]{kingma2014adam}
Diederik~P. Kingma and Jimmy Ba.
\newblock Adam: A method for stochastic optimization.
\newblock In \emph{International Conference on Learning Representations (ICLR)}, 2015.

\bibitem[Liu et~al.(2023)Liu, Zhou, Pan, Wu, Li, Chen, and Zhang]{liu2023curvdrop}
Yang Liu, Chuan Zhou, Shirui Pan, Jia Wu, Zhao Li, Hongyang Chen, and Peng Zhang.
\newblock {CurvDrop}: A {Ricci} curvature based approach to prevent graph neural networks from over-smoothing and over-squashing.
\newblock In \emph{Proceedings of the ACM Web Conference 2023}, pages 221--230, 2023.

\bibitem[Lohit and Jones(2022)]{lohit2022emdipot}
Suhas Lohit and Michael Jones.
\newblock Model compression using optimal transport.
\newblock In \emph{WACV}, pages 2764--2773, 2022.

\bibitem[Lu et~al.(2024)Lu, Guan, Zhao, and Yang]{lu2024adagmlp}
Weigang Lu, Ziyu Guan, Wei Zhao, and Yaming Yang.
\newblock {AdaGMLP}: {AdaBoosting} {GNN}-to-{MLP} knowledge distillation.
\newblock In \emph{Proceedings of the 30th ACM SIGKDD Conference on Knowledge Discovery and Data Mining (KDD '24)}, pages 2060--2071, 2024.
\newblock \doi{10.1145/3637528.3671699}.

\bibitem[Lv et~al.(2024)Lv, Yang, and Li]{chen2024wkd}
Jiaming Lv, Haoyuan Yang, and Peihua Li.
\newblock Wasserstein distance rivals {Kullback-Leibler} divergence for knowledge distillation.
\newblock In \emph{NeurIPS}, 2024.

\bibitem[Nguyen et~al.(2023)Nguyen, Hieu, Nguyen, Ho, Osher, and Nguyen]{nguyen2023revisiting}
Khang Nguyen, Nong~Minh Hieu, Vinh~Duc Nguyen, Nhat Ho, Stanley Osher, and Tan~Minh Nguyen.
\newblock Revisiting over-smoothing and over-squashing using {Ollivier--Ricci} curvature.
\newblock In \emph{ICML}, pages 25956--25979, 2023.

\bibitem[Ni et~al.(2019)Ni, Lin, Luo, and Gao]{ni2019community}
Chien-Chun Ni, Yu-Yao Lin, Feng Luo, and Jie Gao.
\newblock Community detection on networks with {R}icci flow.
\newblock \emph{Scientific Reports}, 9:\penalty0 9984, 2019.

\bibitem[Ollivier(2009)]{ollivier2009ricci}
Yann Ollivier.
\newblock Ricci curvature of {Markov} chains on metric spaces.
\newblock \emph{Journal of Functional Analysis}, 256\penalty0 (3):\penalty0 810--864, 2009.

\bibitem[Peyr{\'e} and Cuturi(2019)]{peyre2019computational}
Gabriel Peyr{\'e} and Marco Cuturi.
\newblock Computational optimal transport with applications to data sciences.
\newblock \emph{Foundations and Trends in Machine Learning}, 11\penalty0 (5-6):\penalty0 355--607, 2019.

\bibitem[Ramp{\'{a}}{\v{s}}ek et~al.(2022)Ramp{\'{a}}{\v{s}}ek, Galkin, Dwivedi, Luu, Wolf, and Beaini]{rampavsek2022recipe}
Ladislav Ramp{\'{a}}{\v{s}}ek, Michael Galkin, Vijay~Prakash Dwivedi, Anh~Tuan Luu, Guy Wolf, and Dominique Beaini.
\newblock Recipe for a general, powerful, scalable graph transformer.
\newblock In \emph{Advances in Neural Information Processing Systems 35}, 2022.

\bibitem[Shchur et~al.(2018)Shchur, Mumme, Bojchevski, and G{\"u}nnemann]{shchur2018pitfalls}
Oleksandr Shchur, Maximilian Mumme, Aleksandar Bojchevski, and Stephan G{\"u}nnemann.
\newblock Pitfalls of graph neural network evaluation.
\newblock \emph{arXiv:1811.05868}, 2018.

\bibitem[Sia et~al.(2019)Sia, Jonckheere, and Bogdan]{sia2019ollivier}
Jayson Sia, Edmond Jonckheere, and Paul Bogdan.
\newblock {Ollivier--Ricci} curvature-based method to community detection in complex networks.
\newblock \emph{Scientific Reports}, 9\penalty0 (1):\penalty0 9800, 2019.

\bibitem[Sreejith et~al.(2016)Sreejith, Mohanraj, Jost, Saucan, and Samal]{sreejith2016forman}
R.~P. Sreejith, Karthikeyan Mohanraj, J{\"u}rgen Jost, Emil Saucan, and Areejit Samal.
\newblock Forman curvature for complex networks.
\newblock \emph{Journal of Statistical Mechanics: Theory and Experiment}, 2016\penalty0 (6):\penalty0 063206, 2016.

\bibitem[Tian et~al.(2023)Tian, Zhang, Guo, Zhang, and Chawla]{tian2022learning}
Yijun Tian, Chuxu Zhang, Zhichun Guo, Xiangliang Zhang, and Nitesh Chawla.
\newblock Learning {MLPs} on graphs: A unified view of effectiveness, robustness, and efficiency.
\newblock In \emph{ICLR}, 2023.

\bibitem[Topping et~al.(2022)Topping, Di~Giovanni, Chamberlain, Dong, and Bronstein]{topping2022understanding}
Jake Topping, Francesco Di~Giovanni, Benjamin~Paul Chamberlain, Xiaowen Dong, and Michael~M Bronstein.
\newblock Understanding over-squashing and bottlenecks on graphs via curvature.
\newblock In \emph{ICLR}, 2022.

\bibitem[Wang et~al.(2021)Wang, Lyu, Li, Xia, Yang, Wang, Wang, Cui, Yang, Sun, and Guo]{wang2021apan}
Xuhong Wang, Ding Lyu, Mengjian Li, Yang Xia, Qi~Yang, Xinwen Wang, Xinguang Wang, Ping Cui, Yupu Yang, Bowen Sun, and Zhenyu Guo.
\newblock {APAN}: Asynchronous propagation attention network for real-time temporal graph embedding.
\newblock In \emph{Proceedings of the 2021 International Conference on Management of Data}, pages 2628--2638, 2021.

\bibitem[Wang et~al.(2023)Wang, Hooi, Liu, and Shah]{wang2023graph}
Yiwei Wang, Bryan Hooi, Yozen Liu, and Neil Shah.
\newblock Graph explicit neural networks: Explicitly encoding graphs for efficient and accurate inference.
\newblock In \emph{WSDM}, 2023.

\bibitem[Wang et~al.(2025)Wang, Zhang, Zhang, and Ye]{yang2024simmlp}
Zehong Wang, Zheyuan Zhang, Chuxu Zhang, and Yanfang Ye.
\newblock Training {MLP}s on graphs without supervision.
\newblock In \emph{Proceedings of the Eighteenth ACM International Conference on Web Search and Data Mining}, pages 697--706, 2025.

\bibitem[Wu et~al.(2023{\natexlab{a}})Wu, Lin, Huang, Fan, and Li]{wu2023extracting}
Lirong Wu, Haitao Lin, Yufei Huang, Tianyu Fan, and Stan~Z. Li.
\newblock Extracting low-/high- frequency knowledge from graph neural networks and injecting it into {MLP}s: An effective {GNN}-to-{MLP} distillation framework.
\newblock In \emph{Proceedings of the AAAI Conference on Artificial Intelligence}, volume~37, pages 10351--10360, 2023{\natexlab{a}}.

\bibitem[Wu et~al.(2023{\natexlab{b}})Wu, Lin, Huang, and Li]{wu2023krd}
Lirong Wu, Haitao Lin, Yufei Huang, and Stan~Z. Li.
\newblock Quantifying the knowledge in {GNN}s for reliable distillation into {MLP}s.
\newblock In \emph{ICML}, 2023{\natexlab{b}}.

\bibitem[Wu et~al.(2024)Wu, Liu, Lin, Huang, and Li]{wu2024hgmd}
Lirong Wu, Yunfan Liu, Haitao Lin, Yufei Huang, and Stan~Z. Li.
\newblock Teach harder, learn poorer: Rethinking hard sample distillation for {GNN}-to-{MLP} knowledge distillation.
\newblock In \emph{CIKM}, pages 2554--2563, 2024.

\bibitem[Yang et~al.(2021)Yang, Liu, and Shi]{yang2021extract}
Cheng Yang, Jiawei Liu, and Chuan Shi.
\newblock Extract the knowledge of graph neural networks and go beyond it: An effective knowledge distillation framework.
\newblock In \emph{Proceedings of the Web Conference 2021}, pages 1227--1237, 2021.

\bibitem[Yang et~al.(2024)Yang, Tian, Xu, Liu, Hong, Qu, Zhang, Cui, Zhang, and Leskovec]{yang2023vqgraph}
Ling Yang, Ye~Tian, Minkai Xu, Zhongyi Liu, Shenda Hong, Wei Qu, Wentao Zhang, Bin Cui, Muhan Zhang, and Jure Leskovec.
\newblock {VQG}raph: Rethinking graph representation space for bridging {GNN}s and {MLP}s.
\newblock In \emph{ICLR}, 2024.

\bibitem[Yang et~al.(2016)Yang, Cohen, and Salakhutdinov]{yang2016revisiting}
Zhilin Yang, William Cohen, and Ruslan Salakhutdinov.
\newblock Revisiting semi-supervised learning with graph embeddings.
\newblock In \emph{ICML}, 2016.

\bibitem[Zhang et~al.(2022)Zhang, Liu, Sun, and Shah]{zhang2022graphless}
Shichang Zhang, Yozen Liu, Yizhou Sun, and Neil Shah.
\newblock Graph-less neural networks: Teaching old {MLP}s new tricks via distillation.
\newblock In \emph{ICLR}, 2022.

\bibitem[Zhou et~al.(2025)Zhou, Ding, Shi, Li, and Shen]{zhou2025tined}
Ziang Zhou, Zhihao Ding, Jieming Shi, Qing Li, and Shiqi Shen.
\newblock Tined: Gnns-to-mlps by teacher injection and dirichlet energy distillation.
\newblock In \emph{International Conference on Machine Learning}, pages 78616--78632. PMLR, 2025.

\end{thebibliography}

\appendix
\setlength{\textfloatsep}{9pt plus 2pt minus 4pt}
\setlength{\intextsep}{8pt plus 2pt minus 2pt}
\setlength{\floatsep}{8pt plus 2pt minus 2pt}
\setlength{\abovecaptionskip}{5pt}

\section{Notation and Preliminaries}
\suppressfloats[t]
\label{app:notation}

\paragraph{Graph and feature notation.}
We use throughout:
\begin{itemize}[leftmargin=*,topsep=2pt,itemsep=0pt]
\item $\mathcal{G} = (\mathcal{V}, \mathcal{E})$ undirected; $d_\mathcal{G}$
the shortest-path metric on $\mathcal{V}$.
\item For $w \in \mathcal{V}$: feature $\mathbf{x}_w \in \mathbb{R}^F$,
hidden state $\mathbf{h}_w \in \mathbb{R}^H$, pre-softmax logits
$\mathbf{z}_w \in \mathbb{R}^C$, and degree $|\mathcal{N}(w)|$.
\item Closed $k$-hop neighborhood:
$\mathcal{N}_k^+(v) := \{w \in \mathcal{V} : d_\mathcal{G}(v, w) \le k\}$,
with $\mathcal{N}^+(v) := \mathcal{N}_1^+(v) = \{v\} \cup \mathcal{N}(v)$.
\item Lazy random-walk distribution with laziness $\alpha = 1/2$:
$\mu_v := \tfrac{1}{2}\delta_v + \tfrac{1}{2|\mathcal{N}(v)|}\sum_{u \in \mathcal{N}(v)} \delta_u$.
\item Mean aggregation (under any feature map $\phi$):
$\bar{\mathbf{x}}_v := \mathbb{E}_{u \sim \mu_v}[\phi(\mathbf{x}_u)]$.
\item Ollivier--Ricci curvature on edge $(u, v) \in \mathcal{E}$:
$\kappa(u, v) := 1 - W_1^{d_\mathcal{G}}(\mu_u, \mu_v)$.
\end{itemize}

\paragraph{Kantorovich--Rubinstein duality.}
For any probability measures $\nu_1, \nu_2$ on $(\mathcal{V}, d_\mathcal{G})$,
\begin{equation}
W_1^{d_\mathcal{G}}(\nu_1, \nu_2)
\;=\;
\sup_{\mathrm{Lip}_{d_\mathcal{G}}(f) \le 1}
\bigl[\mathbb{E}_{\nu_2}[f] - \mathbb{E}_{\nu_1}[f]\bigr],
\label{eq:kr}
\end{equation}
where $\mathrm{Lip}_{d_\mathcal{G}}(f) := \sup_{a \neq b}|f(a) - f(b)|/d_\mathcal{G}(a, b)$.

\paragraph{Standing assumption.}
Throughout the proofs we assume \textbf{(A1)}: the feature map
$w \mapsto \mathbf{x}_w$ (or, more generally, a Lipschitz embedding $\phi$
of node features) is $L$-Lipschitz with respect to $d_\mathcal{G}$:
$\|\phi(\mathbf{x}_a) - \phi(\mathbf{x}_b)\|_2 \le L \cdot d_\mathcal{G}(a, b)$
for all $a, b \in \mathcal{V}$.

\section{Full GNTK Recursion and Proofs}
\suppressfloats[t]
\label{app:gntk}

\subsection{Full GNTK Recursion}
\label{app:gntk-recursion}

For an infinitely wide $K$-layer GNN, the full GNTK recursion
following~\citet{du2019graph} is as follows. Initialize
$\Sigma^{(0)}(v,u) = \mathbf{x}_v^\top \mathbf{x}_u / F$ and
$\Theta^{(0)}(v,u) = \Sigma^{(0)}(v,u)$. For each layer
$k = 1, \ldots, K$, both kernels are first aggregated over neighborhood pairs,
\begin{align}
    \bar{\Sigma}^{(k)}(v,u)
    &=
    \frac{1}{|\mathcal{N}^+(v)||\mathcal{N}^+(u)|}
    \sum_{v' \in \mathcal{N}^+(v)}
    \sum_{u' \in \mathcal{N}^+(u)}
    \Sigma^{(k-1)}(v',u'), \label{eq:gntk-sigma-rec} \\
    \bar{\Theta}^{(k)}(v,u)
    &=
    \frac{1}{|\mathcal{N}^+(v)||\mathcal{N}^+(u)|}
    \sum_{v' \in \mathcal{N}^+(v)}
    \sum_{u' \in \mathcal{N}^+(u)}
    \Theta^{(k-1)}(v',u'), \nonumber
\end{align}
and the nonlinearity is then applied,
\begin{equation}
    \Theta^{(k)}(v,u)
    =
    \Sigma^{(k)}(v,u)
    +
    \dot{\Sigma}^{(k)}(v,u)\,\bar{\Theta}^{(k)}(v,u),
    \label{eq:gntk-theta-rec}
\end{equation}
where $\Sigma^{(k)}(v,u) = \mathbb{E}_{f \sim \mathcal{N}(0,\bar{\Lambda}^{(k)}_{vu})}
[\sigma(f_v)\sigma(f_u)]$ is the post-activation covariance and
$\dot{\Sigma}^{(k)}(v,u) = \mathbb{E}_{f \sim \mathcal{N}(0,\bar{\Lambda}^{(k)}_{vu})}
[\sigma'(f_v)\sigma'(f_u)]$ its derivative kernel (arc-cosine kernels
for ReLU~\citep{cho2009kernel}), with $\bar{\Lambda}^{(k)}_{vu}$ the
$2\times 2$ matrix formed by $\bar{\Sigma}^{(k)}(v,v)$,
$\bar{\Sigma}^{(k)}(v,u)$, and $\bar{\Sigma}^{(k)}(u,u)$.
The MLP NTK $\mathbf{K}^{\mathrm{MLP}}$ is the special case in which both
aggregations are replaced by the identity,
$\bar{\Sigma}^{(k)}(v,u) = \Sigma^{(k-1)}(v,u)$ and
$\bar{\Theta}^{(k)}(v,u) = \Theta^{(k-1)}(v,u)$, so
$\mathbf{K}^{\mathrm{MLP}}(v,u)$ depends only on
$\mathbf{x}_v, \mathbf{x}_u$.

\subsection{Proof of Lemma~\ref{lem:locality} (Aggregation Locality)}
\label{app:proof-locality}

We make precise the structural claim of Lemma~\ref{lem:locality}: the
GNTK entry $\Theta^{(K)}(v,u)$ depends on node features only through the
$K$-hop closed neighborhoods of $v$ and $u$, while the MLP NTK
$\mathbf{K}^{\mathrm{MLP}}(v,u)$ depends only on $\mathbf{x}_v$ and
$\mathbf{x}_u$ themselves. Their difference $\Delta^{(K)}(v,u) :=
\Theta^{(K)}(v,u) - \mathbf{K}^{\mathrm{MLP}}(v,u)$ is therefore the
graph-induced contribution generated by paths in the $K$-hop receptive
fields of $v$ and $u$. The kernel need not be sparse, since the
input covariance $\Sigma^{(0)}$ is generically nonzero for any pair.

\begin{proof}
We prove the support claim by induction on the layer index $k$:
\begin{equation*}
\Sigma^{(k)}(v,u) \text{ and } \Theta^{(k)}(v,u)
\text{ depend on features only through }
\{\mathbf{x}_w : w \in \mathcal{N}_k^+(v) \cup \mathcal{N}_k^+(u)\}.
\end{equation*}

\textbf{Base case ($k=0$).}
$\Sigma^{(0)}(v,u) = \mathbf{x}_v^\top \mathbf{x}_u / F$ is a function of
$\mathbf{x}_v$ and $\mathbf{x}_u$ alone, i.e.\ of features in
$\mathcal{N}_0^+(v) \cup \mathcal{N}_0^+(u) = \{v, u\}$. The same holds
for $\Theta^{(0)} = \Sigma^{(0)}$.

\textbf{Inductive step.}
Suppose the claim holds at layer $k-1$. Fix any pair $(v, u)$. By
Eq.~\eqref{eq:gntk-sigma-rec}, $\bar{\Sigma}^{(k)}(v,u)$ and
$\bar{\Theta}^{(k)}(v,u)$ average $\Sigma^{(k-1)}(v',u')$ and
$\Theta^{(k-1)}(v',u')$ over
$(v', u') \in \mathcal{N}^+(v) \times \mathcal{N}^+(u)$. By the inductive
hypothesis, each summand depends on features only through
$\mathcal{N}_{k-1}^+(v') \cup \mathcal{N}_{k-1}^+(u')$. Since
$v' \in \mathcal{N}^+(v)$ implies $d_\mathcal{G}(v, v') \le 1$, the
triangle inequality gives
$\mathcal{N}_{k-1}^+(v') \subseteq \mathcal{N}_k^+(v)$, and analogously
$\mathcal{N}_{k-1}^+(u') \subseteq \mathcal{N}_k^+(u)$. Hence
$\bar{\Sigma}^{(k)}(v,u)$ and $\bar{\Theta}^{(k)}(v,u)$, and likewise the
diagonal entries $\bar{\Sigma}^{(k)}(v,v)$ and $\bar{\Sigma}^{(k)}(u,u)$,
depend on features only through
$\mathcal{N}_k^+(v) \cup \mathcal{N}_k^+(u)$. The kernels
$\Sigma^{(k)}(v,u)$ and $\dot{\Sigma}^{(k)}(v,u)$ are deterministic
functions of $\bar{\Lambda}^{(k)}_{vu}$~\citep{cho2009kernel}, so they
inherit this dependence, and so does
$\Theta^{(k)}(v,u) = \Sigma^{(k)}(v,u) + \dot{\Sigma}^{(k)}(v,u)\,\bar{\Theta}^{(k)}(v,u)$
by Eq.~\eqref{eq:gntk-theta-rec}.

\textbf{MLP comparison.}
Replacing both aggregations in Eq.~\eqref{eq:gntk-sigma-rec} by the
identity gives the MLP recursion, in which $\Sigma^{(k)}_{\mathrm{MLP}}(v,u)$
and $\Theta^{(k)}_{\mathrm{MLP}}(v,u)$ depend only on previous-layer
entries at the pairs $(v,v)$, $(v,u)$, and $(u,u)$.
Iterating yields $\mathbf{K}^{\mathrm{MLP}}(v,u) =
\Theta^{(K)}_{\mathrm{MLP}}(v,u)$ as a function of
$\mathbf{x}_v, \mathbf{x}_u$ alone.

The Lemma follows on identifying $\Theta^{(K)}(v,u) =
[\mathbf{K}^{(K)}]_{vu}$ and $\Delta^{(K)}(v,u) = \Theta^{(K)}(v,u) -
\mathbf{K}^{\mathrm{MLP}}(v,u)$: $\Delta^{(K)}$ is determined by features
in $\mathcal{N}_K^+(v) \cup \mathcal{N}_K^+(u)$ but not, in general, by
those at $\{v, u\}$ alone --- unlike $\mathbf{K}^{\mathrm{MLP}}$.
\qedhere
\end{proof}

\subsection{Proof of Proposition~\ref{prop:mismatch} (ORC as a Local
Mismatch Proxy)}
\label{app:proof-mismatch}

\begin{proof}
Fix an edge $(i,j) \in \mathcal{E}$. By definition,
$\bar{\mathbf{x}}_i = \mathbb{E}_{X \sim \mu_i}[\phi(\mathbf{x}_X)]$ and
$\bar{\mathbf{x}}_j = \mathbb{E}_{Y \sim \mu_j}[\phi(\mathbf{x}_Y)]$.

\textbf{Step 1 (Coupling lift).}
Let $\pi$ be any coupling of $(\mu_i, \mu_j)$, i.e.\ a joint
distribution on $\mathcal{V} \times \mathcal{V}$ with marginals $\mu_i$
and $\mu_j$. Linearity of expectation gives
\begin{equation*}
\bar{\mathbf{x}}_i - \bar{\mathbf{x}}_j
\;=\;
\mathbb{E}_{(X, Y) \sim \pi}\!\left[\phi(\mathbf{x}_X) - \phi(\mathbf{x}_Y)\right].
\end{equation*}

\textbf{Step 2 (Jensen + Lipschitz).}
Applying the triangle inequality for the Bochner integral (Jensen for
the convex norm $\|\cdot\|_2$),
\begin{equation*}
\|\bar{\mathbf{x}}_i - \bar{\mathbf{x}}_j\|_2
\;\le\;
\mathbb{E}_{(X, Y) \sim \pi}\!\left[\|\phi(\mathbf{x}_X) - \phi(\mathbf{x}_Y)\|_2\right]
\;\le\;
L \cdot \mathbb{E}_{(X, Y) \sim \pi}\!\left[d_\mathcal{G}(X, Y)\right],
\end{equation*}
where the second inequality uses (A1).

\textbf{Step 3 (Wasserstein and ORC substitution).}
Taking the infimum over couplings $\pi \in \Pi(\mu_i, \mu_j)$:
\begin{equation*}
\|\bar{\mathbf{x}}_i - \bar{\mathbf{x}}_j\|_2
\;\le\;
L \cdot \inf_{\pi \in \Pi(\mu_i, \mu_j)}
\mathbb{E}_\pi[d_\mathcal{G}(X, Y)]
\;=\;
L \cdot W_1^{d_\mathcal{G}}(\mu_i, \mu_j).
\end{equation*}
Substituting the ORC definition $W_1^{d_\mathcal{G}}(\mu_i, \mu_j) =
1 - \kappa(i, j)$ for adjacent $(i, j)$ yields the upper bound in
Eq.~\eqref{eq:two-sided}.

\textbf{Step 4 (Lower bound under (C1)--(C2)).}
Let $\pi^\star$ be the optimal coupling of (C1), so that
$\mathbb{E}_{\pi^\star}[d_\mathcal{G}(X,Y)]=W_1^{d_\mathcal{G}}(\mu_i,\mu_j)$.
By Step~1 and (C1),
\begin{equation*}
\|\bar{\mathbf{x}}_i - \bar{\mathbf{x}}_j\|_2
=
\bigl\|\mathbb{E}_{\pi^\star}[\phi(\mathbf{x}_X) - \phi(\mathbf{x}_Y)]\bigr\|_2
\;\ge\;
\eta\,\mathbb{E}_{\pi^\star}\|\phi(\mathbf{x}_X) - \phi(\mathbf{x}_Y)\|_2 .
\end{equation*}
Since $\pi^\star(X,Y)>0$ almost surely, (C2) gives
$\|\phi(\mathbf{x}_X) - \phi(\mathbf{x}_Y)\|_2 \ge m\,d_\mathcal{G}(X,Y)$,
hence
$\|\bar{\mathbf{x}}_i - \bar{\mathbf{x}}_j\|_2
\ge \eta m\,W_1^{d_\mathcal{G}}(\mu_i,\mu_j) = \eta m\,(1-\kappa(i,j))$.
\qedhere
\end{proof}

\paragraph{On conditions (C1) and (C2).}
Both conditions are sufficient rather than necessary for the two-sided
bound in Eq.~\eqref{eq:two-sided}; the mode-conditional reading of
Section~\ref{sec:theory-curvature} uses the lower bound for the
prediction map under underfit and the upper bound for the
representation map under overfit.

\section{Derivation of the Multi-Granularity Decomposition}
\suppressfloats[t]
\label{app:derivation}

We derive the closed-form decomposition that justifies separating the
joint Gaussian $\mathcal{W}_2^2$ alignment in
Eq.~\eqref{eq:diag-gaussian-w2} into independent feature-side and
logit-side terms, and that motivates the curvature-stratified
weighting Eq.~\eqref{eq:curv-weights-method}.

\paragraph{Joint Gaussian model.}
For each node $v$ and neighbor $u \in \mathcal{N}^+(v)$, stack the
hidden representation and pre-softmax logits:
$\mathbf{r}_u = (\mathbf{h}_u^\top, \mathbf{z}_u^\top)^\top \in
\mathbb{R}^{H + C}$. Fit a Gaussian
$\mathcal{N}(\boldsymbol{\mu}_v, \boldsymbol{\Sigma}_v)$ to
$\{\mathbf{r}_u : u \in \mathcal{N}^+(v)\}$ from empirical moments,
with block-form covariance
\begin{equation}
\boldsymbol{\Sigma}_v
\;=\;
\begin{pmatrix}
\boldsymbol{\Sigma}_v^{hh} & \boldsymbol{\Sigma}_v^{hz} \\
\boldsymbol{\Sigma}_v^{zh} & \boldsymbol{\Sigma}_v^{zz}
\end{pmatrix}.
\label{eq:block-cov}
\end{equation}

\paragraph{Assumption D1 (Negligible cross-covariance).}
We model $\boldsymbol{\Sigma}_v^{hz} \approx \mathbf{0}$. This is a
modeling simplification rather than an empirical fact: the logits are
computed from the hidden representations, so the exact
cross-covariance is not zero in general. The simplification is consistent with the
diagonal surrogate of Eq.~\eqref{eq:diag-gaussian-w2}, which drops all
off-diagonal entries, and it is what allows the two blocks to be
weighted separately.

\paragraph{Assumption D2 (Diagonal within blocks).}
Following WKD~\citep{chen2024wkd}, we assume $\boldsymbol{\Sigma}_v^{hh}$
and $\boldsymbol{\Sigma}_v^{zz}$ are diagonal. This reduces the
$(H+C) \times (H+C)$ matrix square root from cost $O((H+C)^3)$ per
node to coordinate-wise square roots in cost $O(H + C)$.

\paragraph{Decomposition.}
Under (D1)+(D2), the closed-form $\mathcal{W}_2^2$ between two joint
Gaussians on $\mathbb{R}^{H + C}$~\citep{chen2024wkd} factorizes
additively along the hidden / logit blocks:
\begin{equation}
\mathcal{W}_2^2\!\left(\mathcal{N}_T(v),\, \mathcal{N}_S(v)\right)
\;=\;
\underbrace{\mathcal{W}_2^2\!\left(\mathcal{N}_T^h(v),\, \mathcal{N}_S^h(v)\right)}
_{\text{feature-side}}
\;+\;
\underbrace{\mathcal{W}_2^2\!\left(\mathcal{N}_T^z(v),\, \mathcal{N}_S^z(v)\right)}
_{\text{logit-side}},
\label{eq:w2-decomp}
\end{equation}
where each diagonal-Gaussian term reduces to
$\|\boldsymbol{\mu}_T - \boldsymbol{\mu}_S\|_2^2 +
\|\sqrt{\boldsymbol{\sigma}_T} - \sqrt{\boldsymbol{\sigma}_S}\|_2^2$.
This is the additive structure exploited in
Eq.~\eqref{eq:diag-gaussian-w2} of the main body and consumed in the
curvature-adaptive objective Eq.~\eqref{eq:caw-loss}.

\paragraph{Implementation note.}
Decomposition Eq.~\eqref{eq:w2-decomp} separates the two blocks but does
not by itself prescribe their relative weighting. The curvature-adaptive
weights $\alpha_v, \beta_v$ in
Eq.~\eqref{eq:curv-weights-method} apply a sign-opposite stratification
of the two terms ($\beta_v \propto e^{+\gamma_h \widetilde{\kappa}_v}$
emphasizes feature-side alignment at high-curvature interior nodes;
$\alpha_v \propto e^{-\gamma_z \widetilde{\kappa}_v}$ emphasizes
logit-side alignment at low-curvature boundary nodes), mirroring the
two failure-mode regimes characterized in
Section~\ref{sec:theory-curvature}. The sign convention therefore follows the theoretical analysis for node classification, and the only
curvature-temperature hyperparameters introduced here are
$\gamma_h, \gamma_z$.

\section{Training Algorithm}
\suppressfloats[t]
\label{app:algorithm}

Algorithm~\ref{alg:dgd} summarizes the two-stage G$^2$MLP training pipeline.
Stage 1 trains a standard GNN teacher and precomputes node-level
Ollivier--Ricci curvature once as a graph-level preprocessing step.
Stage 2 trains the MLP student under the curvature-adaptive objective.
At inference time, only the student MLP is deployed; no graph access,
neighborhood fetching, or sampling is required.

\begin{algorithm}[h]
\caption{G$^2$MLP Training Pipeline}
\label{alg:dgd}
\begin{algorithmic}[1]
\raggedright
\REQUIRE Graph $\mathcal{G} = (\mathcal{V}, \mathcal{E}, \mathbf{X})$,
teacher GNN $T$, student MLP $S_\phi$,\\
joint-loss weights $\lambda_{\mathrm{CE}}, \lambda_{\mathrm{KD}}, \lambda_{\mathrm{CAW}}$,
curvature temperatures $\gamma_h, \gamma_z$
\STATE \textbf{Stage 1: Teacher pre-training and graph preprocessing}
\STATE Train teacher $T$ with cross-entropy until validation convergence
\STATE Precompute node-level ORC $\{\kappa_v\}_{v \in \mathcal{V}}$
       (cached to disk)
\STATE Cache teacher soft labels $\{\hat{\mathbf{y}}_v^T\}$ and
       hidden representations $\{\mathbf{h}_v^T\}$
\STATE \textbf{Stage 2: Student training}
\FOR{epoch $= 1, \ldots, E$}
    \STATE Forward student: $\mathbf{h}_v^S = S_\phi(\mathbf{x}_v)$,
           $\mathbf{z}_v^S = $ classifier head$(\mathbf{h}_v^S)$
    \STATE Compute curvature-adaptive weights:
           $\beta_v \propto e^{+\gamma_h \widetilde{\kappa}_v}$,
           $\alpha_v \propto e^{-\gamma_z \widetilde{\kappa}_v}$
           (Eq.~\eqref{eq:curv-weights-method})
    \STATE Compute curvature-adaptive Wasserstein loss
           $\mathcal{L}_{\mathrm{CAW}}$ (Eq.~\eqref{eq:caw-loss})
    \STATE Compute full objective $\mathcal{L} = \lambda_{\mathrm{CE}}\mathcal{L}_{\mathrm{CE}}
           + \lambda_{\mathrm{KD}} \mathcal{L}_{\mathrm{KD}}
           + \lambda_{\mathrm{CAW}} \cdot \mathcal{L}_{\mathrm{CAW}} /
           (\operatorname{sg}(\mathcal{L}_{\mathrm{CAW}}) + \epsilon_{\mathrm{sg}})$
           (Eq.~\eqref{eq:final-g2mlp})
    \STATE Update $\phi$ by minimizing $\mathcal{L}$ with Adam
\ENDFOR
\STATE \textbf{Inference:} deploy $S_\phi$ only (no graph access required)
\end{algorithmic}
\end{algorithm}

The stop-gradient normalization $\operatorname{sg}(\mathcal{L}_{\mathrm{CAW}})
+ \epsilon_{\mathrm{sg}}$ in line~10 keeps the CAW gradient direction
unchanged but rescales its magnitude to match entropy-scaled losses,
stabilizing the loss scale across datasets with different feature
norms and reducing the sensitivity of $\lambda_{\mathrm{CAW}}$
selection.

\section{Datasets, Baselines, and Implementation Details}
\suppressfloats[t]
\label{app:exp-details}

\subsection{Datasets and Splits}
\label{app:splits}

We evaluate on the six node-classification benchmarks summarized in
Table~\ref{tab:dataset-stats}. Following CPF~\citep{yang2021extract}
and GLNN~\citep{zhang2022graphless}, the citation
networks~\citep{yang2016revisiting} and the Amazon co-purchase
graphs~\citep{shchur2018pitfalls} are used as the largest connected
component of each published graph, and the table reports these
graphs; the OGB graph follows the standard release
of~\citet{hu2020open}. The homophily ratio $h$ is the edge-level
fraction $h = |\{(u,v) \in \mathcal{E} : y_u = y_v\}| / |\mathcal{E}|$.

\begin{table}[!htb]
\centering
\caption{Statistics of the benchmark graphs. $\bar{d}$ is the mean
degree; $h$ is edge homophily.}
\label{tab:dataset-stats}
\small
\begin{tabular}{@{}l rrrr cc@{}}
\toprule
Dataset & \#Nodes & \#Edges & \#Feat. & \#Cls. & $\bar{d}$ & $h$ \\
\midrule
Cora           &   2{,}485 &    5{,}069 & 1{,}433 &  7 &  4.1 & 0.80 \\
Citeseer       &   2{,}110 &    3{,}668 & 3{,}703 &  6 &  3.5 & 0.74 \\
Pubmed         &  19{,}717 &   44{,}324 &      500 &  3 &  4.5 & 0.80 \\
A-computer     &  13{,}381 &  245{,}778 &      767 & 10 & 36.7 & 0.78 \\
A-photo        &   7{,}487 &  119{,}043 &      745 &  8 & 31.8 & 0.83 \\
ogbn-arxiv     & 169{,}343 & 1{,}166{,}243 & 128 & 40 & 13.7 & 0.65 \\
\bottomrule
\end{tabular}
\end{table}

The benchmarks differ in their curvature profiles. Under exact ORC,
the mean edge curvature is negative on all five CPF graphs: most strongly on the citation graphs
Pubmed ($-0.29$) and Cora ($-0.17$), and closest to zero on the dense
co-purchase graph A-photo ($-0.02$, with about half of its edges
positively curved); Citeseer ($-0.09$) and A-computer ($-0.08$) lie in
between.

\paragraph{Splits and evaluation protocol.}
For Cora, Citeseer, Pubmed, A-computer, and A-photo, we use the
\textbf{CPF} split protocol of~\citet{yang2021extract}: 10 random splits
per dataset with the standard CPF train/validation/test ratios; we
report mean$\pm$std over the 10 splits. For ogbn-arxiv, we use the
canonical OGB time-based split~\citep{hu2020open} (training on papers
published until 2017, validation on 2018, and test on 2019 onward) and vary 10 seeds for
model initialization and dropout.

\subsection{Baselines}
\label{app:baselines}

We compare against the canonical GNN-to-MLP distillation baselines
grouped by inference-time graph access. For all baselines we use the
standard student configurations released by the original authors and
refrain from inference-time tricks such as enlarged hidden dimensions
or increased depth.

\textbf{Vanilla baseline.}
\emph{MLP}: the student architecture with identical optimization (no
distillation signal).

\textbf{Graph-free distillation students.}
\emph{GLNN}~\citep{zhang2022graphless}: soft-label KL distillation with
no architectural change to the student.
\emph{KRD}~\citep{wu2023krd}: adds reliability-based per-node sample
weighting on top of GLNN; reliability $r_v = 1 - \delta_v^p$ is
estimated from the teacher's noise sensitivity following the released
protocol.
\emph{FF-G2M}~\citep{wu2023extracting}: distillation of the low- and
high-frequency components of the teacher's knowledge.

\subsection{G$^2$MLP Implementation and Hyperparameters}
\label{app:g2mlp-hparams}

\paragraph{Teacher and student architectures.}
Unless stated otherwise (Appendix~\ref{app:other-teachers}), the GraphSAGE
teacher uses the GCN aggregator. On Cora, Citeseer,
Pubmed, A-computer, and A-photo it has 2 layers (learning rate
$10^{-2}$, weight-decay $5{\times}10^{-4}$, dropout 0, fan-out $5{,}5$
for neighbor sampling); on ogbn-arxiv it follows the OGB-standard
3-layer configuration (hidden dimension 256, dropout 0.2, weight-decay
0, batch normalization, fan-out $5{,}10{,}15$). It is trained for 200
epochs. The MLP student has the same depth and hidden dimension as the
teacher (2 layers and 128 on the five CPF datasets, 3 layers and 256 on
ogbn-arxiv); it is trained for 500 epochs with Adam under the combined
objective in Eq.~\eqref{eq:final-g2mlp}. Both use early stopping on
validation accuracy (patience 50). Per-dataset student optimizer settings (learning rate
$\eta$, weight-decay, dropout) follow the GLNN-released configurations
(Table~\ref{tab:student-optim}); these are held fixed across all our
G$^2$MLP runs and are not retuned.

\begin{table}[!htb]
\caption{Student MLP optimizer hyperparameters.}
\label{tab:student-optim}
\centering
\small
\begin{tabular}{@{}l ccc@{}}
\toprule
Dataset & $\eta$ & weight-decay & dropout \\
\midrule
Cora        & $10^{-2}$        & $5{\times}10^{-3}$ & 0.6 \\
Citeseer    & $10^{-2}$        & $10^{-3}$          & 0.1 \\
Pubmed      & $5{\times}10^{-3}$ & 0                 & 0.4 \\
A-computer  & $10^{-3}$        & $2{\times}10^{-3}$ & 0.3 \\
A-photo     & $5{\times}10^{-3}$ & $2{\times}10^{-3}$ & 0.3 \\
ogbn-arxiv  & $10^{-2}$        & 0                  & 0.2 \\
\bottomrule
\end{tabular}
\end{table}

\paragraph{Hyperparameter search.}
The only G$^2$MLP-specific knobs we sweep are
\begin{align*}
\lambda_{\mathrm{CAW}} &\in \{10^{-2},\, 10^{-1},\, 5{\times}10^{-1}\}, \\
\textsc{caw\_transform} &\in \{\textsc{none},\, \textsc{student},\, \textsc{both}\},
\end{align*}
set per dataset. All
remaining knobs are fixed across all datasets:
$\lambda_{\mathrm{KD}} = 1$, $\lambda_{\mathrm{CE}} = 0$ (pure
distillation), $\gamma_h = \gamma_z = 1$, temperature $\tau = 1$. ORC
is computed with lazy-walk parameter $\alpha = 0.5$, and the
detach-normalized $L_{\mathrm{CAW}}$ uses
$\epsilon_{\mathrm{sg}} = 10^{-8}$. On the Amazon graphs a jointly
trained, hidden-dimension-preserving \texttt{Linear} + \texttt{ReLU}
head transforms $\mathbf{h}_v^S$ only for the $\mathcal{W}_2$ term and
is discarded at inference, leaving the classifier path unchanged; under
an identical configuration, removing it changes accuracy by at most
0.34\,pp (A-computer) and 0.06\,pp (A-photo).
\subsection{Component Ablation}
\label{app:ablation}

We ablate G$^2$MLP along two axes: (i)~$W_2$ alignment vs.\ KL
distillation under uniform weights, and (ii)~which side of the
representation is curvature-stratified. All variants share the same
MLP architecture, optimizer, and 10-seed protocol as the main
results.

\begin{table}[!htb]
\caption{Component ablation on transductive node classification.
Baseline reproduced from Table~\ref{tab:transductive}.}
\label{tab:ablation}
\centering
\small
\begin{tabular}{@{}lccc@{}}
\toprule
Variant & Cora & Citeseer & Pubmed \\
\midrule
GLNN baseline                       & 80.26$\pm$1.66 & 71.22$\pm$1.50 & 75.59$\pm$2.46 \\
$+\;W_2$ uniform (no curvature)     & 80.86$\pm$2.30 & 71.07$\pm$2.19 & 75.22$\pm$2.15 \\
$+\;$curvature, logit-side only     & 81.78$\pm$1.55 & 73.87$\pm$2.09 & 77.27$\pm$2.92 \\
$+\;$curvature, feature-side only   & 81.34$\pm$1.88 & 73.99$\pm$2.02 & 77.14$\pm$2.93 \\
\textbf{Full G$^2$MLP (both sides)} & \textbf{82.54$\pm$1.87} & \textbf{74.51$\pm$2.35} & \textbf{78.17$\pm$2.75} \\
\bottomrule
\end{tabular}
\end{table}

$W_2$ alignment alone is neutral relative to GLNN
(Table~\ref{tab:ablation}); single-side
curvature stratification yields most of the improvement
($+1.5$ to $+2.8$ over GLNN); the two sides are
complementary, validating the sign-opposite stratification predicted
by Section~\ref{sec:theory-curvature}.

\subsection{Compute Environment}
\label{app:compute}

\paragraph{Hardware.}
All node-classification, Graph Transformer, and link-prediction
experiments run on a single NVIDIA GeForce RTX~3090 (24\,GB GDDR6X,
Ampere) on a shared SLURM cluster, with an Intel Xeon Gold~6136
host (3.0\,GHz, dual-socket) and 256\,GB system RAM. 

\paragraph{Software and code.}
Ubuntu~24.04 LTS (kernel 6.8); CUDA~12.4 / cuDNN~9.1; Python~3.10.20;
PyTorch~2.6.0; DGL~2.5.0; \texttt{GraphRicciCurvature}~\citep{ni2019community}
for ORC precomputation. Code is available at
\url{https://github.com/CHENZhewei928/G2MLP-NeurIPS26}.

\paragraph{Wall-clock cost.}
For the original experimental suite, approximate wall-clock costs are:
$\approx$80 GPU-hours for the node-classification experiments
(6 datasets $\times$ 2 settings $\times$ 5 methods $\times$ 10 seeds);
$\approx$40 GPU-hours for the Graph Transformer experiments (see
Appendix~\ref{app:gt-experiments}); $\approx$15 GPU-hours for the
link-prediction experiments (see Appendix~\ref{app:lp-ablation}).
Inference-latency measurements (Figure~\ref{fig:efficiency}) follow the
standard layer-wise full-inference protocol on the same hardware.

\section{Additional Experiments}
\suppressfloats[t]
\label{app:additional}

The experiments in this appendix use the protocol of
Appendix~\ref{app:exp-details} unless stated otherwise.

\subsection{Curvature-Binned Diagnostics}
\label{app:diagnostics}

\paragraph{Node-level attribution.}
Let $\widetilde{\mathbf{H}}^T=\mathbf{U}\boldsymbol{\Sigma}\mathbf{V}^\top$
be the thin SVD of the demeaned teacher representation scaled to unit
trace, and $\widetilde{\mathbf{H}}^S$ the student representation
scaled the same way. With
$\rho_i=\mathbf{u}_i^\top\widetilde{\mathbf{H}}^S\widetilde{\mathbf{H}}^{S\top}\mathbf{u}_i/\sigma_i^2$,
the mismatch attributed to node $v$ is
\begin{equation*}
d_v^{\mathrm{under}}=\sum_{\rho_i<1}(1-\rho_i)^2\sigma_i^2\,u_i(v)^2,
\qquad
d_v^{\mathrm{over}}=\sum_{\rho_i>1}(\rho_i-1)^2\sigma_i^2\,u_i(v)^2 ,
\end{equation*}
which sum over nodes to the global underfit and overfit mismatch.
Nodes are binned into five equal-population quintiles of mean incident
ORC. The students are the vanilla-KL student and G$^2$MLP distilled
from the Table~\ref{tab:transductive} SAGE teachers (5 seeds; mean and
standard error). Table~\ref{tab:energy} reports the totals of both
channels (the bands of Figure~\ref{fig:rank-mismatch}c follow the
dominant channel; Pubmed's mean rank gap is within one standard error
of zero), and Figure~\ref{fig:binned} the overfit channel of both
students per quintile. G$^2$MLP reduces the dominant channel
on four of the five datasets, by 15.5\% to 49.1\%, with Cora flat.

\begin{table}[!htb]
\caption{Total underfit and overfit energy ($\times 10^{-4}$, 5 seeds)
of the vanilla-KL student (GLNN) and G$^2$MLP with the relative change;
the dominant channel of each dataset is in bold.}
\label{tab:energy}
\centering
\small
\begin{tabular}{@{}l l cc cc@{}}
\toprule
 & & \multicolumn{2}{c}{underfit} & \multicolumn{2}{c}{overfit} \\
\cmidrule(lr){3-4}\cmidrule(lr){5-6}
Dataset & regime & GLNN & G$^2$MLP & GLNN & G$^2$MLP \\
\midrule
Cora & underfit & \textbf{3.39} & \textbf{3.41 ($+0.6\%$)} & 0.49 & 0.48 ($-2.9\%$) \\
Citeseer & underfit & \textbf{2.46} & \textbf{2.08 ($-15.5\%$)} & 1.10 & 0.54 ($-50.9\%$) \\
Pubmed & overfit & 0.15 & 0.11 ($-22.6\%$) & \textbf{36.8} & \textbf{27.8 ($-24.5\%$)} \\
A-computer & overfit & 3.44 & 2.78 ($-19.1\%$) & \textbf{34.6} & \textbf{17.6 ($-49.1\%$)} \\
A-photo & overfit & 4.71 & 4.73 ($+0.5\%$) & \textbf{54.7} & \textbf{45.0 ($-17.7\%$)} \\
\bottomrule
\end{tabular}
\end{table}

\paragraph{Edge-level test of the lower bound.}
Eq.~\eqref{eq:two-sided} predicts, under student boundary
under-response, that the teacher--student prediction mismatch
across an edge grows as the edge curvature decreases. For an edge $e$
let $U_e=\|\Delta_e\mathbf{z}^T-\Delta_e\mathbf{z}^S\|_2^2$ be the
squared disagreement of the prediction contrasts across $e$, and
$\Delta_U=\mathbb{E}[U_e\mid\text{lowest-}\kappa\text{ quintile}]/
\mathbb{E}[U_e\mid\text{highest-}\kappa\text{ quintile}]$ the ratio
between the extreme edge quintiles. For vanilla-KL students of the
Table~\ref{tab:transductive} SAGE teachers on the two
underfit-dominated graphs (5 seeds, mean and 95\% interval),
$\Delta_U=1.29\pm0.14$ on Cora and $1.94\pm0.66$ on Citeseer: the
ratio exceeds one with the interval excluding one on both graphs, and
likewise under independently trained GCN teachers. The
node-level underfit energy is a basis-projected quantity and is spread
across quintiles, whereas the bound concerns an edge-level boundary quantity, and at the
edge level the low-curvature concentration appears in $U_e$.

\subsection{Comparison Beyond the Fixed-Teacher Protocol}
\label{app:beyond}

Table~\ref{tab:beyond} reports the methods compared outside the
protocol of Tables~\ref{tab:transductive} and~\ref{tab:inductive}
(conditions (a) and (b) of Section~\ref{sec:setup}).
HGMD-M re-implements HGMD's loss inside our pipeline with our SAGE
checkpoints, splits, and student; HGMD (released) runs the authors'
code with its own teacher, splits, and student. NOSMOG relaxes (a),
since its student consumes DeepWalk positional encodings at inference;
NOSMOG (our teachers) runs the official code under our teachers and
splits ($\ddagger$: its released ogbn-Arxiv configuration enables
neither the adversarial nor the representation-distillation
component). The DeepWalk control of Section~\ref{sec:beyond} changes only
the student input. VQGraph
relaxes (b), since its teacher is trained jointly with a codebook, so
student-level differences also reflect a different teacher; its row
quotes the reproduction by TINED~\citep{zhou2025tined} under the same
GraphSAGE-teacher protocol. AdaGMLP~\citep{lu2024adagmlp} deploys a
boosted ensemble of $K$ MLP students, so its inference cost scales
with $K$; its row runs the official student code with our teachers,
splits, and student width, with its method hyperparameters selected
on validation per dataset.

\begin{table}[!htb]
\caption{Transductive accuracy (\%) of methods outside the protocol of
Table~\ref{tab:transductive}, mean$\pm$std over 10 seeds; ``---''
marks a cell the corresponding pipeline does not provide.}
\label{tab:beyond}
\centering
\footnotesize
\setlength{\tabcolsep}{3.6pt}
\begin{tabular}{@{}l cccccc@{}}
\toprule
Method & Cora & Citeseer & Pubmed & A-computer & A-photo & Arxiv \\
\midrule
G$^2$MLP (Table~\ref{tab:transductive}) & 82.54$\pm$1.87 & 74.51$\pm$2.35 & 78.17$\pm$2.75 & 83.87$\pm$1.09 & 92.85$\pm$0.78 & 64.57$\pm$0.31 \\
\midrule
HGMD-M (matched) & 81.80$\pm$1.61 & 73.23$\pm$2.05 & 77.99$\pm$2.48 & 82.10$\pm$1.27 & 91.47$\pm$1.44 & 63.59$\pm$0.34 \\
HGMD (released) & 82.61$\pm$1.19 & 73.46$\pm$1.06 & 81.03$\pm$0.66 & --- & 92.49$\pm$1.22 & --- \\
\midrule
NOSMOG (our teachers) & 83.07$\pm$1.65 & 74.40$\pm$1.97 & 78.76$\pm$1.58 & 83.65$\pm$1.21 & 92.58$\pm$1.35 & 69.72$\pm$0.33$^{\ddagger}$ \\
NOSMOG (published) & 83.04$\pm$1.26 & 73.78$\pm$1.54 & 77.34$\pm$2.36 & 84.04$\pm$1.01 & 93.36$\pm$0.69 & 71.65$\pm$0.29 \\
\midrule
VQGraph~\citep{zhou2025tined} & 78.66$\pm$1.21 & 74.66$\pm$1.23 & 73.02$\pm$3.51 & 80.16$\pm$2.02 & 92.32$\pm$1.74 & --- \\
\midrule
AdaGMLP (our teachers) & 82.06$\pm$1.91 & 73.80$\pm$1.88 & 77.38$\pm$3.01 & 83.44$\pm$1.21 & 92.40$\pm$1.79 & 65.51$\pm$0.28 \\
\bottomrule
\end{tabular}
\end{table}

\subsection{Other Teacher Architectures}
\label{app:other-teachers}

Table~\ref{tab:other-teachers} reports GCN and GAT teachers on the five
CPF benchmarks under the GLNN protocol. Under the GCN teacher G$^2$MLP improves
over the vanilla-KL student by 0.68\,pp on Cora and 0.33\,pp on
Citeseer, and under the GAT teacher by 1.35\,pp on Cora and 0.72\,pp on
Pubmed; every remaining GCN and GAT cell is within noise (at most
0.26\,pp in magnitude). Together with GraphSAGE
(Table~\ref{tab:transductive}) and the three Graph Transformer teachers
(Table~\ref{tab:gt2mlp}), the curvature-stratified loss improves
distillation under six teacher architectures without re-tuning its
mechanism.

\begin{table}[!htb]
\caption{Transductive accuracy (\%) with GCN and GAT teachers (10
seeds; student hidden width matched to the teacher). Bold marks the
better student; the A-photo cells under GAT are tied.}
\label{tab:other-teachers}
\centering
\small
\begin{tabular}{@{}ll ccccc@{}}
\toprule
Teacher & Student & Cora & Citeseer & Pubmed & A-computer & A-photo \\
\midrule
\multirow{2}{*}{GCN} & GLNN & 81.38$\pm$2.03 & 72.08$\pm$2.08 & \textbf{76.62$\pm$3.03} & 83.22$\pm$2.06 & \textbf{92.98$\pm$0.70} \\
 & G$^2$MLP & \textbf{82.06$\pm$1.38} & \textbf{72.41$\pm$1.69} & 76.36$\pm$2.87 & \textbf{83.37$\pm$1.99} & 92.82$\pm$0.89 \\
\midrule
\multirow{2}{*}{GAT} & GLNN & 80.25$\pm$2.75 & 72.07$\pm$1.85 & 75.89$\pm$2.61 & \textbf{82.95$\pm$1.10} & 92.86$\pm$1.16 \\
 & G$^2$MLP & \textbf{81.60$\pm$1.84} & \textbf{72.22$\pm$1.55} & \textbf{76.61$\pm$2.83} & 82.92$\pm$1.26 & 92.86$\pm$1.21 \\
\bottomrule
\end{tabular}
\end{table}

\begin{figure}[!htb]
  \centering
  \includegraphics[width=0.66\linewidth]{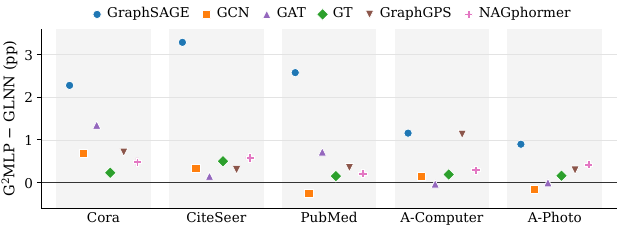}
  \caption{Test-accuracy gain of the G$^2$MLP student over the
  vanilla-KL student (pp) for six teacher architectures (10 seeds;
  Tables~\ref{tab:transductive}, \ref{tab:gt2mlp},
  and~\ref{tab:other-teachers}); the horizontal line marks equal
  accuracy.}
  \label{fig:teachers}
\end{figure}

\subsection{Curvature Computation Ablation}
\label{app:curvabl}

The common-neighbor approximation replaces the per-edge optimal
transport by
$W_1(\mu_i,\mu_j)\approx\alpha+(1-\alpha)\bigl(1-|N(i)\cap N(j)|/\sqrt{d_i d_j}\bigr)$,
which costs $O(1)$ per edge and is the curvature used on ogbn-Arxiv
and ogbn-Products. Swapping the curvature computation among exact ORC,
this approximation, augmented Forman, classic 1d Forman, and balanced
Forman~\citep{sreejith2016forman,topping2022understanding} under an
identical training recipe (10 seeds on each of the five CPF benchmarks
per curvature), the five computations agree within 0.3\,pp on average,
with a maximum single-dataset deviation of 1.25\,pp, so any curvature
that preserves the ordering from bridges to communities suffices for
the allocation mechanism.

\subsection{Robustness to Feature Outliers}
\label{app:outlier}

We replace a fraction $p$ of the node-feature entries with heavy-tailed
Cauchy noise (scale equal to one feature standard deviation) and
retrain every method from scratch at each level (5 seeds; Cora and
A-computer as the underfit- and overfit-dominated regimes).
Figure~\ref{fig:outlier} shows the result. G$^2$MLP tracks the
vanilla-KL student within noise at every level (largest gap 0.67\,pp
in either direction), so the diagonal-moment surrogate introduces no
measurable heavy-tail fragility relative to plain logit distillation;
on A-computer, where the representation-level term carries the most
weight, G$^2$MLP loses 10.6\,pp from clean to $p=20\%$ against
14.4\,pp for the plain MLP. Even at $p=20\%$, both distilled students
keep an advantage of about 15\,pp over the feature-only MLP on both
datasets. The
theory does not cover strongly contaminated or multimodal
neighborhoods, which we state as a limitation in
Appendix~\ref{app:limitations}.

\begin{figure}[!htb]
  \centering
  \includegraphics[width=0.87\linewidth]{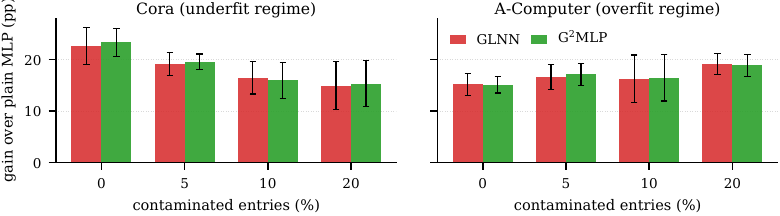}
  \caption{Mean test-accuracy gain of each distilled student over the
  feature-only MLP (pp) as a fraction of the feature entries is
  replaced by Cauchy noise (5 seeds; bars: std of the per-seed gain).}
  \label{fig:outlier}
\end{figure}

\section{Graph Transformer Teacher Experiments}
\suppressfloats[t]
\label{app:gt-experiments}

This appendix consolidates implementation and analysis details for the
Graph Transformer (GT) to MLP distillation experiments in
Section~\ref{sec:gt2mlp}.

\subsection{GT Architectures and Configurations}
\label{app:gt-arch}

We evaluate three GT architectures as teachers, matching the columns of
Table~\ref{tab:gt2mlp} in the main body: vanilla
\texttt{GT}~\citep{dwivedi2020generalization},
\texttt{GraphGPS}~\citep{rampavsek2022recipe}, and
\texttt{NAGphormer}~\citep{chen2023nagphormer}. Each teacher is trained
for 500 epochs with Adam and uses the same hidden dimension as our
GraphSAGE teacher, so that students are compared at matched width; the
remaining hyperparameters are set per dataset. The GraphGPS teacher uses dense
$N \times N$ attention as its canonical default; we additionally
analyze a sparse edge-softmax variant in
Appendix~\ref{app:gps-sparse-vs-dense}. Across all GT teachers, the
student is a 2-layer MLP with hidden dimension matched to the teacher
and is trained with the G$^2$MLP objective of the main experiments; the
CAW weight and the curvature temperature are chosen per (teacher,
dataset) cell from a small grid.

\subsection{Why Students Can Exceed Their GT Teachers}
\label{app:gt-reversal}

Table~\ref{tab:gt2mlp} shows that the MLP student meets or exceeds its
GT teacher on $13/15$ cells. This $S \geq T$ reversal arises because
GTs at the dataset sizes considered (2K--20K nodes) are difficult to
train well: the small label budget (140 / 120 / 60 nodes on
Cora / Citeseer / Pubmed) provides limited supervision for $O(N^2)$
attention interactions, so randomly initialized full attention often
produces near-uniform patterns that the 500-epoch training schedule
does not escape. Distillation acts as a regularizer in this regime: the
student inherits the teacher's prediction targets while shedding the
teacher's overfit attention noise, with the curvature-adaptive
weighting in G$^2$MLP further focusing the alignment budget on
neighborhoods where the teacher's geometry is reliable. The same
phenomenon has been observed for ``born-again'' distillation in vision
under analogous data-scarce regimes~\citep{furlanello2018born}; our
contribution is showing that it persists with GT teachers and that the
gain extends to MLP-only inference.

\subsection{Why NAGphormer Is Often the Strongest GT Teacher}
\label{app:why-nagphormer}

NAGphormer is the strongest teacher on most CPF benchmarks because its hop-token
embedding produces a permutation-invariant summary at every hop
distance, which (i) regularizes the attention pattern when the training
signal is small, and (ii) injects graph structure into the token
sequence rather than relying on attention alone to recover it.

\subsection{GraphGPS: Sparse vs.\ Dense Attention}
\label{app:gps-sparse-vs-dense}

GraphGPS canonically uses dense $N \times N$ attention; sparse
edge-softmax is a common alternative that restricts attention to graph
edges. On the dataset sizes considered ($N \in [2{,}000, 20{,}000]$),
both variants fit comfortably in 24\,GB GPU memory and converge in
500 epochs. The dense vs.\ sparse comparison in
Table~\ref{tab:gps-comparison} is therefore not a memory tradeoff but
an inductive-bias one: edge-softmax acts as an explicit locality prior,
while dense attention must learn locality from scratch.

\begin{table*}[!htb]
\centering
\caption{G$^2$MLP vs.\ GLNN with the GraphGPS teacher under both
canonical dense ($N \times N$) and sparse (edge-softmax) attention.
Conventions as in Table~\ref{tab:gt2mlp}: $T$ teacher,
$S$ GLNN student, $S^{\dagger}$ G$^2$MLP student. \textbf{Bold}: best graph-free student per teacher.}
\label{tab:gps-comparison}
\small
\begin{tabular}{@{}l ccc c ccc@{}}
\toprule
& \multicolumn{3}{c}{GraphGPS (dense, canonical)}
& & \multicolumn{3}{c}{GraphGPS (sparse, edge-softmax)} \\
\cmidrule(lr){2-4} \cmidrule(lr){6-8}
Dataset & $T$ & $S$ & $S^{\dagger}$ & & $T$ & $S$ & $S^{\dagger}$ \\
\midrule
Cora &76.21$\pm$1.91 & 76.06$\pm$1.99  & \textbf{76.77$\pm$1.97}  & &78.12$\pm$2.12 & 77.86$\pm$2.09  & \textbf{78.74$\pm$1.87} \\
Citeseer &68.07$\pm$2.62 & 69.38$\pm$2.30  & \textbf{69.68$\pm$2.05}  & &68.88$\pm$2.25 & 70.29$\pm$2.13  & \textbf{70.44$\pm$2.13} \\
Pubmed &71.77$\pm$2.81 & 71.68$\pm$2.61  & \textbf{72.03$\pm$2.58}  & &75.61$\pm$2.13 & \textbf{76.08$\pm$2.34}  & 75.97$\pm$2.21 \\
A-computer &78.30$\pm$1.70 & 78.99$\pm$1.73  & \textbf{80.12$\pm$1.55}  & &81.84$\pm$0.88 & 82.05$\pm$1.18  & \textbf{82.55$\pm$0.86} \\
A-photo &90.41$\pm$1.10 & 91.03$\pm$0.95  & \textbf{91.32$\pm$0.93}  & &90.34$\pm$1.21 & 91.12$\pm$1.04  & \textbf{91.26$\pm$1.02} \\
\bottomrule
\end{tabular}
\end{table*}

\paragraph{Discussion.}
The dense GraphGPS teacher is intrinsically weaker than vanilla GT or
NAGphormer on small citation graphs, for the reason given in
Appendix~\ref{app:gt-reversal}; the MPNN branch of GraphGPS carries
most of the model's capacity. Even so, with the
dense teacher G$^2$MLP improves over GLNN and surpasses the teacher on
all five datasets; with the sparse teacher it surpasses the teacher on
all five and improves over GLNN on four (Pubmed: $-0.11$\,pp). The sparse variant is a stronger teacher
than dense on four of the five graphs (e.g.\ A-computer $T = 81.84$
vs.\ $78.30$; Cora $T = 78.12$ vs.\ $76.21$; A-photo is a tie,
$90.34$ vs.\ $90.41$). The G$^2$MLP gains over GLNN hold in nine of the ten cells across the
two attention modalities.

\section{Link Prediction: Pure-Distillation Ablation}
\suppressfloats[t]
\label{app:lp-ablation}

This appendix accompanies Section~\ref{sec:lp} of the main body, which
reports only the multi-task regime (KD + edge-BCE). Here we ablate the
edge-BCE term to isolate the contribution of embedding distillation
alone, and analyze the resulting cross-regime gap.

\paragraph{Setup.}
We use the largest-connected-component CPF graphs of
Table~\ref{tab:dataset-stats}. Teacher and student architectures follow
Appendix~\ref{app:g2mlp-hparams}. Edges are split
85/5/10 into train/validation/test with the standard
\texttt{RandomLinkSplit} protocol and 1:1 negative sampling; all
models train for up to 500 epochs with best-validation-AUC
checkpointing. Per-dataset MLP
optimizer settings are those of Table~\ref{tab:student-optim}. KRD's
reliability weights use 10 noise realizations of
magnitude 1.0, bin count 50, and momentum 0.99~\citep{wu2023krd}. The
teacher is trained once per seed
and shared across both regimes.

\paragraph{G$^2$MLP-LP hyperparameter search.}
G$^2$MLP-LP introduces three task-specific hyperparameters in addition
to the curvature temperatures $\gamma_h, \gamma_z$ shared with node
classification: the KD weight $\lambda_{\mathrm{KD}}$, the CAW weight
$\lambda_{\mathrm{CAW}}$, and the sign of $\gamma_h$ in the
feature-side weight $\beta_v$ of Eq.~\eqref{eq:curv-weights-method}.
The node-classification analysis
(Section~\ref{sec:theory-curvature}) motivates the
$\beta_v \propto e^{+\gamma_h \widetilde{\kappa}_v}$ convention for
label-fitting tasks; link prediction is governed by a different
inductive signal (inner-product edge ranking), so we treat the sign
as task-specific and set it per dataset. We refer to the two
choices as the \emph{community} direction
($\beta_v \propto e^{+\gamma_h \widetilde{\kappa}_v}$, emphasizing
high-curvature interior nodes) and the \emph{bridge} direction
($\beta_v \propto e^{-\gamma_h \widetilde{\kappa}_v}$, emphasizing
low-curvature boundary nodes). The logit-side direction $\alpha_v$
is held fixed at its node-classification convention
($\alpha_v \propto e^{-\gamma_z \widetilde{\kappa}_v}$, emphasizing
boundary nodes) for all LP experiments.
We select per dataset from a 72-config grid on the multi-task split:
$\lambda_{\mathrm{KD}} \in \{0.05, 0.1, 0.2, 0.5, 1, 2\}$,
$\lambda_{\mathrm{CAW}} \in \{10^{-3}, 10^{-2}, 10^{-1}\}$,
$\gamma_h \in \{1, 10\}$, $\beta_v$ direction $\in \{\text{community}, \text{bridge}\}$.
Other settings are held fixed at the validated configuration from
node-classification (Appendix~\ref{app:g2mlp-hparams}). Selected
values:
\begin{center}
\small
\begin{tabular}{@{}l cccc@{}}
\toprule
Dataset & $\lambda_{\mathrm{KD}}$ & $\lambda_{\mathrm{CAW}}$ & $\gamma_h$ & $\beta_v$ direction \\
\midrule
Cora     & 0.20  & $10^{-2}$  & 10 & bridge \\
Citeseer & 0.05  & $10^{-3}$  & 10 & community \\
Pubmed   & 0.05  & $10^{-2}$  & 1  & community \\
\bottomrule
\end{tabular}
\end{center}
The dataset-dependent shape is informative. Cora prefers a strong
distillation weight ($\lambda_{\mathrm{KD}}=0.2$) and the
\emph{bridge} direction, consistent with its sparse,
high-curvature-variance structure where bridge nodes carry
disproportionate LP-relevant information. Citeseer and Pubmed prefer
weaker distillation and the \emph{community} direction, where
neighborhood Gaussians are denser and statistically more reliable, and
LP supervision benefits from interior alignment. Baselines (MLP, GLNN,
KRD) use a single global hyperparameter set; G$^2$MLP-LP's per-dataset
selection mirrors KRD's own per-dataset $p$-power tuning in the
original paper~\citep{wu2023krd}.

\paragraph{Pure regime: ablating the edge-BCE term.}
In the pure-distillation regime the student sees only the teacher's
frozen embeddings (MSE loss); no edge labels are available. On Citeseer
and Pubmed the GLNN / KRD pure-KD students already exceed the teacher,
clearly on Citeseer ($84.26$ AUC vs.\ $81.21$) and marginally on Pubmed
($88.63$ vs.\ $88.36$), showing that the teacher's embeddings carry
essentially sufficient LP information; the edge-BCE term in the
multi-task default changes Citeseer AUC by less than one point and
adds $1.4$--$1.7$ AUC on Pubmed. On Cora,
however, the GLNN / KRD pure students collapse to $56$--$57$ AUC, a
$\approx$25-point drop from the teacher's $81.29$: the teacher's
embeddings encode label-discriminative directions but lose the
geometric anisotropy needed for inner-product LP scoring on this small,
sparse graph. \emph{G$^2$MLP-LP eliminates this collapse}: with
curvature-adaptive reweighting and neighborhood Gaussian
$\mathcal{W}_2$ alignment, the pure student recovers to $81.77$ AUC on
Cora ($+24.5$ over KRD; matches the teacher), and stays competitive
with KRD on the other two graphs. This is direct evidence that
curvature-aware geometry transfers more LP-relevant structure than
uniform or reliability-weighted MSE.

\begin{table}[!htb]
\caption{Link prediction (pure-distillation regime, no edge-BCE).
Conventions as in Table~\ref{tab:lp-multi}.}
\label{tab:lp-pure}
\centering
\small
\resizebox{\textwidth}{!}{%
\begin{tabular}{@{}l ll ll ll@{}}
\toprule
& \multicolumn{2}{c}{Cora} & \multicolumn{2}{c}{Citeseer} & \multicolumn{2}{c}{Pubmed} \\
\cmidrule(lr){2-3} \cmidrule(lr){4-5} \cmidrule(lr){6-7}
Method & AUC & AP & AUC & AP & AUC & AP \\
\midrule
SAGE (teacher, ref)
  & 81.29$\pm$1.99 & 80.54$\pm$2.69
  & 81.21$\pm$3.15 & 81.80$\pm$3.87
  & 88.36$\pm$1.65 & 86.83$\pm$1.54 \\
\midrule
GLNN
  & 56.18$\pm$5.79 & 56.81$\pm$4.55
  & 83.99$\pm$2.03 & 84.23$\pm$2.16
  & 88.62$\pm$1.57 & 87.17$\pm$1.58 \\
KRD
  & 57.23$\pm$5.53 & 57.41$\pm$5.02
  & \textbf{84.26$\pm$2.56} & \textbf{84.46$\pm$2.59}
  & 88.63$\pm$1.56 & 87.18$\pm$1.57 \\
\textbf{G$^2$MLP (Ours)}
  & \textbf{81.77$\pm$2.25} & \textbf{80.28$\pm$2.84}
  & 82.89$\pm$3.61 & 83.15$\pm$4.39
  & \textbf{88.82$\pm$1.46} & \textbf{87.43$\pm$1.35} \\
\bottomrule
\end{tabular}%
}
\end{table}

\paragraph{Cross-regime gap.}
Subtracting the pure-regime AUC (Table~\ref{tab:lp-pure}) from the
multi-task AUC (Table~\ref{tab:lp-multi}) on the same teacher and split
yields, for GLNN: $+22.4$ on Cora, $-0.7$ on Citeseer, $+1.4$ on Pubmed;
KRD is almost identical ($+22.0 / -1.0 / +1.7$). Where the gap is large
and positive (Cora), the edge-BCE term is doing genuine work for the
uniform-MSE baselines. G$^2$MLP-LP closes most of the Cora gap on its
own (gap $+2.1$), confirming that
curvature-aware geometry already encodes the structural signal that
uniform MSE needs the BCE term to supply.

\section{Limitations and Future Directions}
\suppressfloats[t]
\label{app:limitations}

\paragraph{Diagonal Gaussian neighborhood surrogate.}
The diagonal Gaussian in Eq.~\eqref{eq:caw-loss} is a scalable
moment-matching surrogate, not an assumption that all graph
neighborhoods are exactly Gaussian
(Section~\ref{sec:gaussian-surrogate}, Assumption~D2 in
Appendix~\ref{app:derivation}). For strongly multimodal neighborhoods
it may under-represent the alignment cost by collapsing several modes
onto one mean and per-coordinate variance.
Mixture or low-rank covariance surrogates are natural extensions.

\paragraph{Scope of the theoretical results.}
Lemma~\ref{lem:locality} establishes aggregation locality, and
Proposition~\ref{prop:mismatch} gives an ORC-dependent two-sided bound
under sufficient conditions; ORC is a proxy for allocating
supervision, not an estimator of kernel eigenvalues. These results do
not prove that all spectral underfit lies at low curvature or all
spectral overfit at high curvature. A spectral curriculum that
schedules supervision from the student's empirical NTK spectrum and the
teacher kernel's energy along it would be more direct; whether it
transfers kernel structure rather than outputs is open.

\paragraph{Heterophily and contaminated features.}
Curvature-stratified alignment presupposes locally concentrated
neighborhoods. Under strong heterophily, where a node's neighbors
belong to different communities, this premise fails and the mechanism
degrades gracefully: it neither helps nor materially hurts relative to
uniform distillation. Under heavy-tailed
feature contamination (Appendix~\ref{app:outlier}) G$^2$MLP likewise
tracks the vanilla-KL student; the theory covers neither case.

\end{document}